\documentclass{article}

\usepackage{microtype}
\usepackage{graphicx}
\usepackage{booktabs} 

\usepackage{hyperref}

\usepackage[accepted]{icml2024}

\usepackage{amsmath}
\usepackage{amssymb}
\usepackage{mathtools}
\usepackage{amsthm}
\usepackage{xspace}

\usepackage[capitalize,noabbrev]{cleveref}
\usepackage{thmtools,thm-restate}
\theoremstyle{plain}
\newtheorem{theorem}{Theorem}[section]

\newtheorem{lemma}[theorem]{Lemma}

\theoremstyle{definition}

\newcommand{\proj}{HEPT\xspace}
\usepackage{enumitem}
\usepackage{multirow}
\usepackage{graphicx}
\usepackage{wrapfig}
\usepackage{caption}
\usepackage{subcaption}

\usepackage{bm}

\newcommand{\wt}[1]{\widetilde{#1}}

\usepackage[textsize=tiny]{todonotes}

\icmltitlerunning{Locality-Sensitive Hashing-Based Efficient Point Transformer}

\begin{document}

\twocolumn[
\icmltitle{Locality-Sensitive Hashing-Based Efficient Point Transformer \\ with Applications in High-Energy Physics}



\icmlsetsymbol{equal}{*}

\begin{icmlauthorlist}
\icmlauthor{Siqi Miao}{gt}
\icmlauthor{Zhiyuan Lu}{bupt}
\icmlauthor{Mia Liu}{pu}
\icmlauthor{Javier Duarte}{ucsd}
\icmlauthor{Pan Li}{gt}
\end{icmlauthorlist}

\icmlaffiliation{gt}{Georgia Institute of Technology}
\icmlaffiliation{bupt}{Beijing University of Posts and Telecommunications}
\icmlaffiliation{pu}{Purdue University}
\icmlaffiliation{ucsd}{University of California San Diego}

\icmlcorrespondingauthor{Siqi Miao}{siqi.miao@gatech.edu}
\icmlcorrespondingauthor{Pan Li}{panli@gatech.edu}

\icmlkeywords{Machine Learning, Geometric Deep Learning, Efficient Transformer, High-Energy Physics, AI for Science}

\vskip 0.3in
]



\printAffiliationsAndNotice{}  

\begin{abstract}
This study introduces a novel transformer model optimized for large-scale point cloud processing in scientific domains such as high-energy physics (HEP) and astrophysics. Addressing the limitations of graph neural networks and standard transformers, our model integrates local inductive bias and achieves near-linear complexity with hardware-friendly regular operations. One contribution of this work is the quantitative analysis of the error-complexity tradeoff of various sparsification techniques for building efficient transformers. Our findings highlight the superiority of using locality-sensitive hashing (LSH), especially OR \& AND-construction LSH, in kernel approximation for large-scale point cloud data with local inductive bias. Based on this finding, we propose LSH-based Efficient Point Transformer (\textbf{\proj}), which combines E$^2$LSH with OR \& AND constructions and is built upon regular computations. \proj demonstrates remarkable performance on two critical yet time-consuming HEP tasks, significantly outperforming existing GNNs and transformers in accuracy and computational speed, marking a significant advancement in geometric deep learning and large-scale scientific data processing. Our code is available at \url{https://github.com/Graph-COM/HEPT}.
\end{abstract}

\vspace{-5mm}
\section{Introduction}\label{sec:intro}
\vspace{-1mm}
Many scientific applications require the processing of complex research objects, often represented as large-scale point clouds --- a set of points within a geometric space --- in real time. 
For instance, in high-energy physics (HEP)~\cite{radovic2018machine}, to search for new physics beyond the standard model, e.g., new particles predicted by supersymmetric theories~\cite{oerter2006theory, wess1992supersymmetry}, the CERN Large Hadron Collider (LHC) produces 1 billion particle collisions per second, forming point clouds of detector measurements with tens of thousands of points~\cite{gaillard2017cern},
necessitating real-time analysis due to storage limitations. Similarly, in astrophysics~\cite{halzen2010invited}, the IceCube Neutrino Observatory records 3,000 events per second using over 5,000 sensors~\cite{aartsen2017icecube}, and simulations of galaxy formation and evolution need to run billions of particles~\cite{nelson2019illustristng}. In drug discovery applications, large-scale real-time computation is crucial for screening billions of protein-antibody pairs, each requiring molecular dynamics simulations of systems with thousands of atoms~\cite{durrant2011molecular}. Facing these extensive computational demands, machine learning, in particular geometric deep learning (GDL), has emerged as a revolutionary tool, offering to replace the most resource-intensive parts of these processes ~\cite{bronstein2017geometric, bronstein2021geometric}.

\begin{figure*}[ht]
    \centering
    \vspace{-0.15cm}
    \includegraphics[trim={0.4cm 0.3cm 0.0cm 0cm},clip,width=0.85\linewidth]{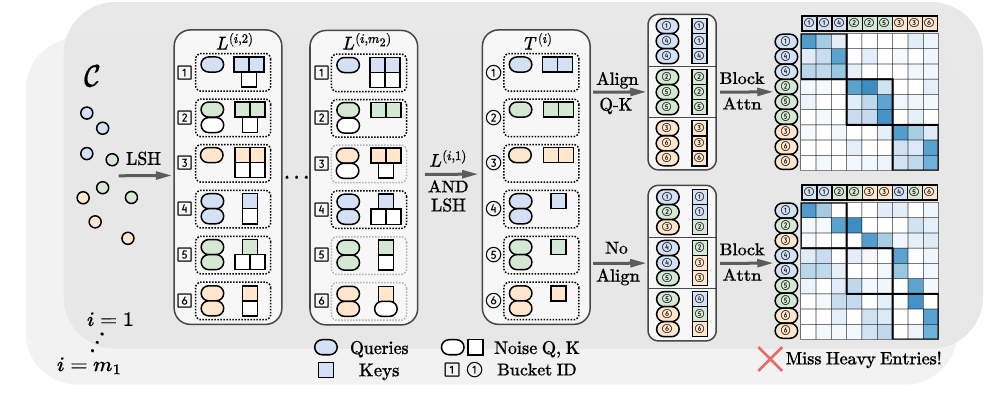}
    \vspace{-0.25cm}
    \caption{Pipeline of \proj. Elements that share the same color represent points from the same local neighborhood.
    \proj employs OR \& AND LSH to minimize noise caused by individual hash functions. \proj also integrates point coordinates as extra AND LSH codes for query-key alignment, maintaining computational regularity without compromising accuracy.
    }
    \label{fig:arch}
    \vspace{-0.5cm}
\end{figure*}

In these scientific applications, inference tasks often exhibit local inductive bias, meaning that the labels to be predicted are primarily determined by aggregating information from local regions within the ambient space. Leveraging local inductive bias can significantly reduce computational complexity. Consequently, graph neural networks (GNNs) have gained widespread use due to their proficiency in exploiting such sparse data patterns~\cite{jing2020learning, Kansal:2021cqp, satorras2021n, dezoort2021charged, abbasi2022graph, li2023semi}. However, GNNs still face two computational challenges that hinder their application in real-time scenarios. \emph{First, the procedure of graph construction is time-consuming:} GNNs often use k-NN or other relational rules to construct graphs~\cite{stark2022equibind, lieret2023high}. Creating these graphs from $n$ points using brute-force methods involves $\mathcal{O}(n^2)$ complexity. While algorithms like KD-trees theoretically offer $\mathcal{O}(n\log n)$ complexity, their limited parallelizability makes them impractical for real-time data processing pipelines~\cite{wieschollek2016efficient}. \emph{Second, The irregular structure of graphs and the neighborhood aggregation process in GNNs lead to irregular computations and random memory access.} These factors, coupled with dynamic computation graphs for different inputs, pose significant computation challenges for conventional hardware architectures~\cite{jang2010exploiting, hashemi2018learning, abadal2021computing}. These issues render GNNs less suitable for large-scale real-time point cloud analysis. Therefore, there is a growing interest in exploring alternative approaches to address the above challenges.

Recently, transformer architectures have demonstrated impressive capabilities across various domains~\cite{vaswani2017attention, brown2020language}. Unlike GNNs, transformers are noted for their ability to model long-range dependencies and their compatibility with hardware due to regular computation patterns. However, a major limitation of standard transformers is their quadratic complexity to input size, which poses challenges for processing large-scale point cloud data. In this study, we aim to integrate the strengths of both transformers and GNNs by developing an efficient transformer model for point cloud processing. This model incorporates local inductive bias and may achieve near-linear complexity, balancing high accuracy with hardware efficiency through regular and parallelizable computations.

Several studies have been conducted on efficient transformers.
However, efficient transformers are not yet fully embraced by scientific domains dealing with geometric data~\cite{Kansal:2022spb, pata2023improved}. The primary issue is that existing methods, which use low-rank~\cite{wang2020linformer} or sparse~\cite{kitaev2020reformer} approximations of the attention matrix, often overlook approximation errors in method design~\cite{kitaev2020reformer, daras2020smyrf} or fail to adequately consider local inductive bias~\cite{choromanski2020rethinking, peng2021random}, leading to undesired model performance. Some methods may compromise approximation accuracy for computational regularity~\cite{kitaev2020reformer, daras2020smyrf, zandieh2023kdeformer, han2023hyperattention}. Moreover, a systematic understanding of the tradeoff between approximation errors and computational complexity among the different methods is missing, making it difficult to select the most effective method for GDL tasks.

In this work, we close the gap by conducting a quantitative analysis of the error-complexity tradeoff, focusing on two widely used techniques for building efficient transformers:
random Fourier features (RFF) ~\cite{rahimi2007random} and locality-sensitive hashing (LSH)~\cite{indyk1998approximate}.
Our analysis indicates that for tasks with local inductive bias, RFF consistently exhibits higher approximation error compared to LSH under subquadratic complexity. We also discover that relying solely on OR-construction LSH results in suboptimal performance, and combining OR \& AND-construction LSH~\cite{leskovec2020mining}, often ignored in prior research, is essential to minimize errors for point clouds in large multidimensional spaces.

Inspired by the analysis, we propose an LS\textbf{H}-based \textbf{E}fficient \textbf{P}oint \textbf{T}ransformer (\textbf{\proj}), designed to support highly regular computations with near-linear complexity and provably low approximation errors for tasks with local inductive bias. \proj leverages a kernel that explicitly embeds local inductive bias for attention calculation, and adopts E$^2$LSH combined with both OR \& AND LSH to effectively minimize approximation errors. 
To ensure computational regularity without compromising accuracy, \proj partitions queries and keys into regular buckets based on their LSH codes, and computes only blockwise attention weights. To address the issue of the misalignment of query-key buckets (Sec.~\ref{sec:recalibrate}),  \proj proposes to integrate point coordinates as extra AND LSH codes. The pipeline of \proj is shown in Fig.~\ref{fig:arch}.

To validate the effectiveness of \proj, we evaluate it on two critical computationally intensive HEP tasks: charged particle tracking~\cite{amrouche2020tracking} and pileup mitigation~\cite{martinez2019pileup}, with their significance elaborated in Appendix~\ref{app:tasks}. \proj is benchmarked against five GNNs and seven efficient transformers adapted from both NLP and CV domains under a unified framework on three datasets (with one of them contributed by us). \proj significantly outperforms all baselines, achieving state-of-the-art (SOTA) accuracy with up to $203\times$ speedup on GPUs. Our experiments also show that existing RFF-based methods fail to deliver competitive performance. LSH-based baselines achieve acceptable accuracy for medium-sized point clouds, however, they struggle to scale to larger datasets with tens of thousands of points, and only \proj can process them efficiently and achieve SOTA prediction accuracy.

\section{Preliminaries}\label{sec:prelim}
\textbf{Geometric Deep Learning Tasks.} We focus on tasks with each sample represented as a point cloud $\mathcal{C} = (\mathcal{V}, \bm{X}, \bm{\rho})$, where $\mathcal{V} = \{u_1, \cdots, u_n\}$ is a set of $n$ points, $\bm{X} \in \mathbb{R}^{n \times k_1}$ includes $k_1$-dimensional features for each point, and $\bm{\rho} \in \mathbb{R}^{n \times k_2}$ specifies the coordinates of these points in a $k_2$-dimensional space. We consider GDL tasks that require learning meaningful $H$-dimensional latent representations for each point via a neural network $f: \mathcal{C} \rightarrow \mathbb{R}^{n\times H}$. Depending on the specific task, these representations are either used for direct point-wise label prediction or point-pair-wise relationship analysis, or whole point cloud prediction via aggregating (e.g., averaging) these representations.

\textbf{Random Fourier Features.} Consider any positive definite shift-invariant kernel $k(\bm{x}, \bm{y}) = k(\bm{x} - \bm{y})$ with $\bm{x}, \bm{y} \in \mathbb{R}^d$ that is properly normalized, i.e., $k(\bm{0}) = 1$. Bochner's Theorem~\cite{rudin1991fourier} guarantees that its Fourier transform $k^*(\bm{w})$ is a probability distribution.
Thus, ~\citet{rahimi2007random} propose RFFs to approximate such a kernel by $k(\bm{x}, \bm{y}) \approx \psi(\bm{x})^\top \psi(\bm{y})$ with $\psi: \mathbb{R}^{d} \rightarrow \mathbb{R}^D$, where
    {\small$\psi(\bm{x}) =  \sqrt{\frac{2}{{D}}}\Big(\sin (\bm{w}_1^{\top} \bm{x}),  \cos (\bm{w}_1^{\top} \bm{x}), \ldots, \sin (\bm{w}_{D/2}^{\top} \bm{x}), \cos (\bm{w}_{D/2}^{\top} \bm{x})\Big)^\top$}
    $, \bm{w}_i \stackrel{i i d}{\sim} k^*(\bm{w}).$

\textbf{Locality-Sensitive Hashing.} LSH~\cite{indyk1998approximate} was proposed for efficient nearest-neighbor search. With high probability, it hashes close data points into the same bucket and distant ones into different buckets. E$^2$LSH~\cite{datar2004locality} is an LSH variant for Euclidean distances with a hash family $\mathcal{H}$ and hash functions $h_{\bm{a}, b}(\bm{x}) \in \mathcal{H}$, where, for a point $\bm{x}\in \mathbb{R}^d$, $h_{\bm{a}, b}(\bm{x}) = \lfloor \frac{\bm{a}\cdot \bm{x} + b}{r} \rfloor$, $\bm{a} \sim \mathcal{N}(0, \bm{I})$, $b \sim \mathcal{U}(0, r)$, and $r > 0$ is a hyperparameter to control bucket sizes.
There are also variants for angular distances~\cite{andoni2015practical} and inner products~\cite{shrivastava2014asymmetric}.
To amplify LSH's performance, AND LSH, OR LSH, or a hybrid of both can be utilized.
AND LSH concatenates multiple (say $m_2$) hash functions $h_j\in \mathcal{H}$ to form a new hash family $\mathcal{G}$,
where for $g \in \mathcal{G}$, $g(\bm{x}) = [h_1(\bm{x}), \ldots, h_{m_2}(\bm{x})]$, and two points are deemed neighbors if they match across all $m_2$ hash functions in $g$.
OR LSH, on the other hand, forms multiple (say $m_1$) hash tables from $\mathcal{G}$, i.e., $g_1(\bm{x}), \ldots, g_{m_1}(\bm{x})$ with each $g_i(\bm{x}) = [h_{i,1}(\bm{x}), \ldots, h_{i,m_2}(\bm{x})]$, and two points are neighbors if they match in any one of these $m_1$ tables.
When $m_2 = 1$ (one hash function per table), it becomes OR-only LSH, and when $m_2 \geq 2$, it is a hybrid of OR \& AND LSH.

\textbf{Efficient Transformers as Kernel Approximation.}
The quadratic complexity of the original transformer~\cite{vaswani2017attention} comes from the computation of self-attention. That is, with $\bm{Q}, \bm{K}, \bm{V} \in \mathbb{R}^{n \times d}$, where each token or point $u$ in the point cloud is associated with a row $\bm{q}_u, \bm{k}_u, \bm{v}_u$ in these matrices, and
$\operatorname{Attn}(\bm{Q}, \bm{K}, \bm{V}) =  \exp( {\bm{Q}\bm{K}^\top} ) \bm{V}$. Here, $\exp( {\bm{Q}\bm{K}^\top})$ is of size $n\times n$, and we omit the normalization terms for simplicity. Viewing the attention as a kernel $\exp({\bm{x}^\top \bm{y}} )$, several methods have been proposed to approximate it for efficiency. Many of these methods are RFF-based~\cite{peng2021random, choromanski2020rethinking, luo2021stable, choromanski2023learning} or LSH-based~\cite{kitaev2020reformer, daras2020smyrf, zandieh2023kdeformer, han2023hyperattention}.
For example, RFFs can be utilized to approximate $\exp({\bm{x}^\top \bm{y}} ) \approx  \widehat{\psi}( {\bm{x}})^\top \widehat{\psi}( {\bm{y}}) $, with, e.g., $\widehat{\psi}(\bm{x}) = \exp(\frac{\|\bm{x}\|^2}{2}) \psi(\bm{x})$~\cite{peng2021random}, reducing the complexity to $\mathcal{O}(n)$. 
As for LSH-based methods, e.g., Reformer~\cite{kitaev2020reformer} equalizes query and key vectors and sets their norms to be $1$, enabling the use of angular distance-based LSH~\cite{andoni2015practical} to efficiently find large entries in the attention matrix $\exp( {\bm{Q}\bm{K}^\top} )$ as its approximation, resulting in $\mathcal{O}(n \log n)$ complexity.

\textbf{Notation.} Later, we use $\tilde{\mathcal{O}}$, $\tilde{\Theta}$, and $\tilde{o}$ denote soft-$\mathcal{O}$, soft-$\Theta$, and soft-$o$ , respectively. They are variants of Big-$\mathcal{O}$, Big-$\Theta$, and Little-$o$ that suppress polylogarithmic factors.

\section{Error-Computation Analysis for RFF/LSH}\label{sec:analysis}
One of the key steps of designing efficient transformers relies on effective kernel approximation. So, this section aims to analyze the tradeoff between the approximation error ($\epsilon$) and computational complexity ($F$) of both RFF- and LSH-based methods in point cloud systems.
Our goal is to enable direct comparisons between RFF-based and LSH-based methods for GDL tasks, where local inductive bias holds, 
seeking to provide theoretical guidance for the design of efficient transformers to be discussed in Sec.~\ref{sec:arch}.
To summarize, we achieve the following insights: Let $\epsilon$ denote the squared error of the attention weight approximation averaged over all point pairs in a system and $F$ denote the total number of floating point operations (FLOPs).

\begin{enumerate}[noitemsep, topsep=0pt, parsep=0pt, partopsep=0pt]
    \item RFF results in an error $\epsilon = \tilde{\Theta}(\frac{n}{F})$, which 
    is consistently worse than LSH under subquadratic complexity, i.e., when $F=\tilde{o}(n^2)$.
    \item LSH is better suited for tasks with local inductive bias, yielding $\epsilon = \tilde{\Theta}(\frac{1}{n})$ via OR-only LSH. However, OR-only LSH finds it hard to further reduce such error if $F$ is set to be almost linear, i.e., $F=\tilde{\mathcal{O}}(n)$.
    \item Utilizing both OR \& AND LSH significantly improves performance. The error $\epsilon = \tilde{\mathcal{O}}(\exp(-\frac{F}{n\text{polylog}(n)})\frac{1}{n})$, which means that $\epsilon$ can be further exponentially reduced by almost linear complexity $F=\tilde{\mathcal{O}}(n)$.
\end{enumerate}

Practitioners primarily interested in the architecture implementation of \proj may choose to skip the rest of this section, and check Sec.~\ref{sec:arch} directly.

\subsection{Characterizing Local Inductive Bias}\label{sec:LIB}
The following notions aim to formally characterize the local inductive bias of a point cloud system of interest.

\begin{restatable}[Bounded-Support Kernels]{definition}{defKernel}
\label{def:kernel}
Consider a properly normalized shift-invariant kernel defined as $k_s(\bm{x}, \bm{y}) = k_s(\bm{x} - \bm{y})$, where $k_s(\bm{x}, \bm{y}) \in [0, 1]$, $s>0$ and $k_s(\bm{0}) = 1$. This kernel exhibits bounded support, i.e., $k_s(\bm{x} - \bm{y}) = 0$ for $\|\bm{x} - \bm{y}\|_2 > s$. For any $\bm{x}, \bm{y} \in \mathbb{R}^d$, the computational complexity of $k_s(\bm{x}, \bm{y})$ is linear in $d$.
\end{restatable}

\begin{restatable}[Local Inductive Bias]{assumption}{defSparsity}
\label{def:sparsity}
Consider a bounded point cloud system $\mathcal{C}$ with $n$ points located at $\{\bm{x}_1,...,\bm{x}_n\}$  in a $d$-dim unit ball, i.e., $\bm{x}_i\in\mathbb{R}^d$ and $\|\bm{x}_i\|_2\leq 1$. Denote the empirical distribution of the point-pair distances as $\phi(z) = \frac{1}{n(n-1)}\sum_{i,j\in [n], i\neq j} \delta_{\|\bm{x}_i - \bm{x}_j\|_2}(z)$ where $\delta_a(\cdot)$ is 1-dim dirac delta function. $\mathcal{C}$ is said to hold local inductive bias if the ground-truth function for the learning task over $\mathcal{C}$ can be approximated by a transformer with full attention matrices whose attention weights can be represented as a bounded-support kernel $k_s$ between point locations $\bm{x}_i$'s, where the bound $s$ satisfies $\int_{0}^{s} \phi(z)dz \sim \tilde{\mathcal{O}}(\frac{1}{n})$.
\end{restatable}

Intuitively, local inductive bias assumes that in a point cloud system, a point primarily interacts with its local neighborhood, where the number of points each point interacts with is on average at most $\mathcal{O}(\text{polylog}(n))$. This assumption means that the optimal full attention matrix has at most $\mathcal{O}(n\cdot\text{polylog}(n))$ non-zeros, which gives the foundation to build efficient transformers with almost linear complexity. 
The challenge lies in how to identify those non-zeros using near-linear complexity and regular operations.

Note that the above-assumed kernel $k_s$ for characterizing local inductive bias can be viewed as an inherent property of the point cloud system and the learning task, which may not necessarily follow the common implementation of attention kernel such as   $\exp({F(\bm{x})^\top G(\bm{y})})$ with some parameterized functions $F,G$.
Although the conventional kernel $\exp({\bm{x}^\top \bm{y}})$ is not strictly with bounded support, with the functions $F,G$, practical attention weights $\exp({F(\bm{x})^\top G(\bm{y})})$ still hold the potential of approximating a bounded support kernel and yield reasonable performance. That having been said, as shown in our experiments, an attention kernel that directly models local inductive bias (see Sec.~\ref{sec:kernel-for-LIB}) often yields better performance for the tasks where local inductive bias indeed exists. 

\textbf{How large could $s$ be in practice?}  Suppose points are almost uniformly allocated in the $d$-dim unit ball, and then, $\int_{0}^s \phi(z) dz = \Theta(s^d)$. In this case, local inductive bias means the point pairs within $s=\tilde{\mathcal{O}}(\frac{1}{n^{1/d}})$ distance hold positive attention weights.  

\subsection{Error-Computation Tradeoff}
\textbf{RFF.} We instantiate our analysis of RFF based on a widely used feature map $\psi(\bm{x})^\top\psi(\bm{y})$ as defined in Sec.~\ref{sec:prelim}, where the complexity $F$ is proportional to the feature dimension $D$. The following theorem indicates that RFF can hardly reduce the error to $\frac{1}{n}$ when $F$ is sub-quadratic in $n$.
\begin{restatable}[$\epsilon-F$ Tradeoff of RFF]{theorem}{theoremRFF}
\label{theorem:RFF}
    Assume $k_s(\bm{x}, \bm{y})$ is positive definite. If approximating it by RFF $\psi(\bm{x})^\top\psi(\bm{y})$ in point cloud systems described in Assumption~\ref{def:sparsity},
    the error $\epsilon=\Theta{(\frac{n d}{F})}$.
\end{restatable}

\textbf{OR-only LSH.} Since many previous works use OR-only LSH, we are to first analyze the approximation error in such a setting. Note that $F$ is proportional to the number of hash tables $m_1$ in this setting and the latter OR \& AND LSH setting. We base our analysis on E$^2$LSH with $r$ as the bucket size defined in Sec.~\ref{sec:prelim}, while the analysis can be similarly extended to other types of hash functions.
To achieve the next theorem,  we need a further assumption that is satisfied as long as the point allocation $\phi(z)$ is not concentrated at $z=a$ for some particular $a\in [0,2)$.  
\begin{restatable}[$\epsilon-F$ Tradeoff of OR-only E$^2$LSH]{theorem}{theoremLSH}
\label{theorem:LSH}
     Assume there exists $r$ such that $\int_{0}^r \phi(z) dz \leq c_1 r$ and $\int_{r}^\infty \frac{1}{z} \phi(z) dz \leq c_2$ for some positive constants $c_1$ and $c_2$. 
    The OR-only E$^2$LSH may achieve $\epsilon = \tilde{\Theta}(\exp \left(-\frac{c_3 F}{d n^2 s}\right) \frac{1}{n})$ where $c_3$ is a positive constant depending on $c_1$ and $c_2$. 
\end{restatable}

Putting Theorem~\ref{theorem:RFF} and~\ref{theorem:LSH} together, 
clearly, OR-only LSH can outperform RFF when $F= \tilde{o}(n^2)$, indicating that LSH is always preferable for subquadratic complexity given point cloud systems with local inductive bias. This is attributed to the fact that LSH tends to zero out kernel values with a high probability for distant pairs.

\textbf{OR \& AND LSH.} OR-only LSH's error dependency on $\exp(-\frac{c_3F}{dn^2s})$ shows that to further effectively reduce the error $ \tilde{\Theta} (\frac{1}{n})$, $F$ has to be in the order of  $ d n^2s$. However, as $s$ could be much larger than $n^{-1}$ in practice (see the discussion in Sec.~\ref{sec:LIB}), this asks for $F$ being super-linear in $n$. 
The issue, due to our analysis, is caused by many distant point pairs being mapped to the same hash bucket if one uses OR-only LSH, which motivates us to inspect the use of OR \& AND LSH. 

\begin{restatable}[$\epsilon-F$ Tradeoff of OR \& AND E$^2$LSH]{theorem}{theoremANDLSH}
\label{theorem:ANDLSH}
    Suppose each hash table contains $m$ hash functions.
    Assume there exists $m$ such that $\int_{0}^{r}\phi(z)dz = \tilde{\mathcal{O}}(\frac{1}{n})$ and $\int_{r}^{\infty}\phi(z)\frac{r^{m}}{z^{m}}dz\leq \int_{0}^{r}(\sqrt{2\pi} - \frac{z}{r})^{m}\phi(z)dz$, where $r=ms$. By choosing such $r$ as the bucket size, the OR \& AND E$^2$LSH may achieve 
        $\epsilon = \tilde{\mathcal{O}}(\exp (- \frac{c_4F}{dn( \operatorname{polylog}(n) + m)})\frac{1}{n}).$ 
\end{restatable}
Note that if we consider systems with almost uniformly allocated points, there exists $m \leq d$ satisfying the assumption. This theorem shows that with OR \& AND LSH combined, $F\sim nd(\operatorname{polylog}(n)+d)$ is sufficient to reduce the error exponentially, necessitating the use of OR \& AND LSH.

\subsection{Numerical Experiments}\label{sec:numerical}

\begin{figure}[t]
\begin{center}
\centerline{\includegraphics[trim={0.0cm 0.4cm 0.0cm 0.0cm},clip,width=0.625\columnwidth]{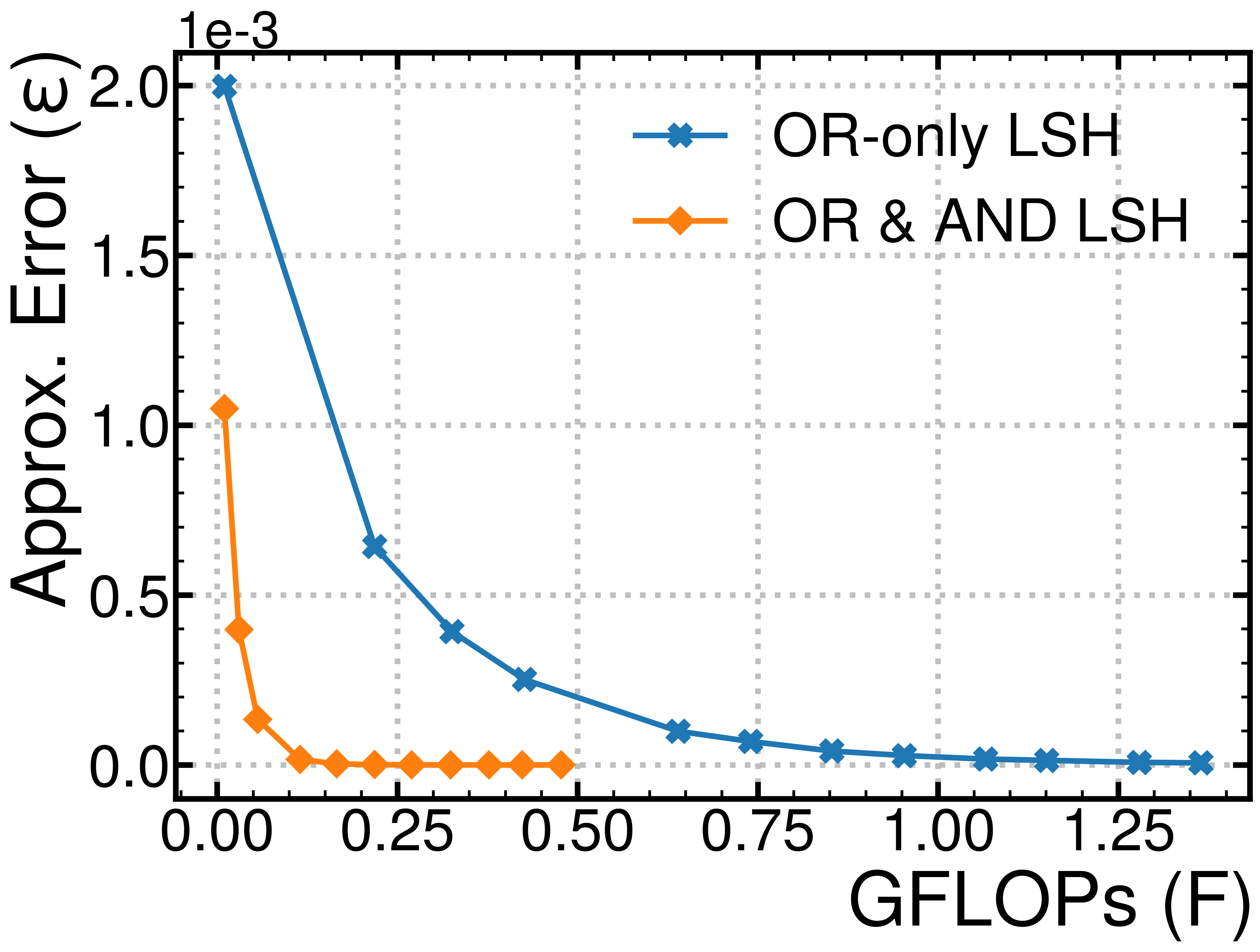}}
\vspace{-0.25cm}
\caption{
The error-computation tradeoff from numerical experiments. OR \& AND LSH decreases the error exponentially with near-linear complexity, validating our analysis.
}
\label{fig:numerical}
\vspace{-0.9cm}
\end{center}
\end{figure}

To further validate the effectiveness of OR \& AND LSH as proved in our theoretical analysis, we conducted additional numerical experiments, and the results are depicted in Fig.~\ref{fig:numerical}.

In this numerical study, we generate $n=30,000$ points uniformly distributed across a 2D square with a side length of $10$. To model local inductive bias, each point interacts only with its $64$ nearest neighbors, approximating a ground-truth kernel value of $\exp \left(-\frac{1}{2}\|\bm{x}-\bm{y}\|^2\right)$ (our theoretical results are not limited to this kernel). Points beyond this neighborhood have a kernel value of $0$. Additional details are provided in Appendix~\ref{app:numerical}. 
In this study, since $n^2$ is roughly of the same magnitude as $1$ GFLOP ($1e9$ FLOPs), Fig.~\ref{fig:numerical} reveals that OR-only LSH can only effectively reduce the error when the computational budget is on the order of $n^2$. Conversely, OR \& AND LSH achieves exponential error reduction with substantially fewer FLOPs, demonstrating its superior efficiency and accuracy.


\vspace{-2mm}
\section{\proj Architecture}\label{sec:arch}
Motivated by our theoretical insights, we propose \proj in this section, which is illustrated in Fig~\ref{fig:arch}. We will first introduce the attention kernel considered, and then describe an approach for approximating it with OR \& AND LSH. Lastly, we present a way to ensure computational regularity without compromising approximation accuracy.
\vspace{-2mm}
\subsection{Kernel with Explicit Local Inductive Bias}\label{sec:kernel-for-LIB}
Given a query-key pair $(\bm{{q}}_u, \bm{{k}}_v)$, we propose to
use the following kernel for attention computation:
$
k(\bm{{q}}_u, \bm{{k}}_v ) = \exp (  -\frac{1}{2} \|  \bm{{q}}_u - \bm{{k}}_v   \|^2   )
$, where $\bm{{q}}_u = [\bm{\wt{q}}_u \| \sqrt{2 \omega} \bm{\rho}_u ]$ and $\bm{{k}}_v = [\bm{\wt{k}}_v \| \sqrt{2 \omega} \bm{\rho}_v ]$ are concatenated from the original transformer's queries and keys $\bm{\wt{q}}_u, \bm{\wt{k}}_v \in \mathbb{R}^d$ with point coordinates $\bm{\rho}_u, \bm{\rho}_v \in \mathbb{R}^{k_2}$ and learnable parameters $\omega \in \mathbb{R}^+$.
The full attention mechanism is then $\operatorname{Attn}(\bm{A}, \bm{V}) =  \bm{D}^{-1} \bm{A} \bm{V}$, with $\bm{A} \in \mathbb{R}^{n \times n}$ comprising elements $\bm{A}_{uv} = k(\bm{{q}}_u, \bm{{k}}_v )$, and $\bm{D} = \operatorname{diag}(\bm{A} \bm{1})$ for normalization, where $\bm{1}$ represents an all-one vector. 
This kernel
enables the use of E$^2$LSH (or RFF) for approximation and allows for explicit modeling of local inductive bias: 
the attention score $k(\bm{{q}}_u, \bm{{k}}_v ) \rightarrow 0$ as $\|  \bm{{q}}_u - \bm{{k}}_v   \|^2$ increases.

Note that \proj can also support efficient computation of the conventional attention kernel $\exp ( \bm{q}_u^T \bm{k}_v)$ by transforming it into $\exp (  -\frac{1}{2} \|  F(\bm{{q}}_u) - G(\bm{{k}}_v)  \|^2)$ for some functions $F,G$~\cite{shrivastava2014asymmetric, daras2020smyrf}. However, this kernel does not work well for the HEP tasks in this work, due to its failure in explicitly modeling local inductive bias.

\begin{figure}[t]
\begin{center}
\centerline{\includegraphics[trim={1.5cm 0.2cm 0.5cm 0.2cm},clip,width=1\columnwidth]{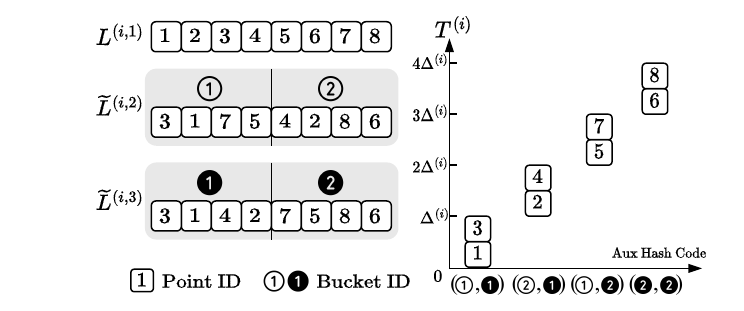}}
\vspace{-0.2cm}
\caption{
The above shows how to obtain AND hash code $T^{(i)}$ with $m_2 = 3, B_{ij} = 2$. 
Points are assumed to be pre-sorted based on their raw hash values with $\min(L^{(i, 1)})=0$.
}
\label{fig:trick}
\vspace{-0.9cm}
\end{center}
\end{figure}

\subsection{OR \& AND LSH for Attention Computation}\label{sec:or-and}
As shown in Sec.~\ref{sec:analysis}, to effectively approximate kernels with local inductive bias, it is critical to utilize OR \& AND LSH for near-linear complexity with guaranteed low approximation errors. Therefore, we propose an architecture that integrates OR \& AND LSH for attention computation.

Specifically, we first construct $m_1$ hash tables (OR LSH), each with $m_2$ hash functions (AND LSH) for each query and key. Each hash function is  E$^2$LSH without bucketization, i.e., $h_{\bm{a}}(\bm{x}) = \bm{a} \cdot \bm{x}$. Consequently, each query $\bm{q}_u$ or key $\bm{k}_v$ yields $m_1 \times m_2$ raw hash values, denoted as ${L}_{\bm{q}_u}^{(ij)}, {L}_{\bm{k}_v}^{(ij)} \in \mathbb{R}$ for $i\in [m_1]$ and $j\in[m_2]$, respectively.
Due to the property of E$^2$LSH, if $\bm{q}_u$ and $\bm{k}_v$ hold small $\|\bm{q}_u-\bm{k}_v\|_2$, they are likely to have close hash values ${L}_{\bm{{q}}_u}^{(ij)}$ and ${L}_{\bm{{k}}_v}^{(ij)}$. 

For each query/key, uniformly denoted as $\bm{z}$, in the $i^{th}$ hash table, our goal is to combine the $m_2$ raw hash values ${L}_{\bm{z}}^{(ij)}$ into a single \emph{\textbf{AND hash code}} $T_{\bm{{z}}}^{(i)} \in \mathbb{R}$
such that query/key pairs with close $|T_{\bm{{q}}_u}^{(i)}-T_{\bm{{k}}_v}^{(i)}|$ must have close $|{L}_{\bm{q}_u}^{(ij)}-{L}_{\bm{k}_v}^{(ij)}|$ for all $j \in [m_2]$, i.e., an AND operation. 
Then, attention can be computed using the resulting AND hash codes $T_{\bm{{z}}}^{(i)}$'s from those $m_1$ hash tables. 
The details are as follows. 

\textbf{Obtaining AND Hash Code ${T}_{\bm{z}}^{(i)}$.}
For each query/key $\bm{{z}}$, we keep ${L}_{\bm{{z}}}^{(i1)}$ as it is as a real-valued 1D \emph{\textbf{base hash code}}, 
and bucketize each ${L}_{\bm{{z}}}^{(ij)}$ for $j \geq 2$, which leads to an $(m_2 - 1)$-sized tuple of integer bucket indices, named as \emph{\textbf{aux hash code}}. 
As illustrated in Fig~\ref{fig:trick},
we compute the AND hash code by 1) allocating $\bm{{z}}$'s only with the same aux hash code to a unique range in $\mathbb{R}$, and 2) within each allocated range, positioning different $\bm{z}$'s according to their real-valued base hash codes. Specifically, to obtain the aux hash code, given $i,j$ and the number of desired buckets $B_{ij}$, sort $L_{\bm{{z}}}^{(ij)}$ for all queries and keys (thus $2n$ in total). Then, starting from 1,  every $\lfloor \frac{2n}{B_{ij}} \rfloor$ consecutive points receives a bucket index denoted as $\wt{L}_{\bm{{z}}}^{(ij)}\in \{1,...,B_{ij}\} $. We use $\wt{L}_{\bm{{z}}}^{(ij)}$ as the aux hash code. Then, the AND hash code ${T}_{\bm{z}}^{(i)}$ can be computed as follows: Let $\Delta^{(i)} = \max (L^{(i1)}) - \min (L^{(i1)})$, 
\begin{align*}
    T^{(i)}_{\bm{z}} &= L^{(i1)}_{\bm{z}} + \Delta^{(i)}\sum_{j=2}^{m_2} \left[(\wt{L}^{(ij)}_{\bm{z}}-1) \prod_{j'=2}^{j-1} B_{ij'}\right].
\end{align*}
This way of computation makes sure that points with the same aux hash codes are assigned to adjacent ranges. Note that to avoid the boundary effect of hash bucketing (two close points being put into two consecutive buckets consistently), we may shift the number of desired buckets $B_{ij}$ for different $i,j$ while guaranteeing the total number of buckets $\prod_{j=2}^{m_2} B_{ij}$ unchanged, which effectively shifts the bucket boundaries, similar to the random shifts in original E$^2$LSH.

\textbf{Merging $m_1$ OR LSH Results.}
By repeating the above steps $m_1$ times, we yield $T_{\bm{q}_u}^{(i)}, T_{\bm{k}_v}^{(i)}$ for every query and key with $i \in [m_1]$, which are used to compute $m_1$ many sparse attention matrices $\bm{A}^{(i)}$ with regular non-zero blocks 
(this process will be elaborated in the next Sec.~\ref{sec:recalibrate}). The final embedding can be computed as $\bm{E} =  \bm{D}^{-1} \bm{A} \bm{V}\in \mathbb{R}^{n \times d}$, where $\bm{A} \bm{V}$ is computed via the sum of $m_1$ regular block matrix multiplications $\sum_i \bm{A}^{(i)} \bm{V}$, and $\bm{D}$ is computed similarly via  $\text{diag}(\sum_i \bm{A}^{(i)} \bm{1})$.

\begin{figure}[t]
    \centering
    \begin{subfigure}[t]{0.44\columnwidth} 
        \centering
        \includegraphics[trim={1.1cm 0.85cm 0.0cm 0.1cm}, clip, width=0.7\linewidth]{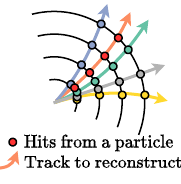}
        \vspace*{-2mm}
        \caption{Charged Particle Tracking}
        \label{fig:tracking}
    \end{subfigure}
    \hspace{0.01\columnwidth}
    \begin{subfigure}[t]{0.44\columnwidth}
        \centering
        \includegraphics[trim={0.15cm 0.0cm 0.3cm 0.3cm}, clip,width=0.66\linewidth]{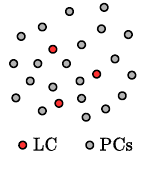}
        \vspace*{-2mm}
        \caption{Pileup Mitigation}
        \label{fig:pp}
    \end{subfigure}
    \vspace*{-2mm}
    \caption{Illustrations of the two HEP tasks.}
    \label{fig:datasets}
    \vspace{-0.3cm}
\end{figure}

\vspace{-1mm}
\subsection{Regular Computation  with Query-Key Alignment}\label{sec:recalibrate}
Here, we elaborate the way to compute sparse attention matrix $\bm{A}^{(i)}$ with regular non-zero blocks via regular operations. 
A naive way to compute $\bm{A}^{(i)}$ is to grouping queries and keys with similar $T_{\bm{q}_u}^{(i)}, T_{\bm{k}_v}^{(i)}$ into buckets and compute their attention. However, this is inefficient because of the potentially non-uniform allocation between queries' and keys' hash codes. The numbers of queries and the numbers of keys may be very different in the same buckets and may shift significantly across the buckets, of which the attention computation needs significantly irregular computations. 

To introduce regularity, we separately process queries and keys. We partition queries into equal-sized buckets by truncating their AND hash codes $T_{\bm{q}_u}^{(i)}$ for $u\in[n]$, and partition the keys in the same way. Then, we compute the attention between queries and keys with the same bucket index, which essentially computes a block-diagonal attention matrix and is highly regular and hardware-friendly.

However, in practice, we observe that the above way to bucketize queries and keys separately may introduce a misalignment issue between queries and keys and miss those query-key pairs with large attention values, since it may include queries and keys with rather different hash codes in the same bucket, as illustrated with an example in Fig.~\ref{fig:arch}. Reformer~\cite{kitaev2020reformer} circumvents this issue by tying queries and keys, i.e., $\bm{q}_u = \bm{k}_u$, but this limits its modeling capacity. We address this challenge by integrating point coordinates as extra AND hash codes, detailed as follows.

\textbf{Point Coordinates as Extra AND Hash Codes.}
In GDL tasks with point cloud data, we may leverage spatial proximity to align query buckets and key buckets. Specifically, we can obtain $d'$ additional AND hash values based on the coordinates of point $u$, $\bm{\rho}_u \in \mathbb{R}^{k_2}$  (typically $d' \leq k_2$)  for aux hash code computation. These hash values are shared by both queries and keys, i.e., ${L}_{\bm{{q}}_{u}}^{({i(m_2+\ell)})} = {L}_{\bm{{k}}_u}^{({i(m_2+\ell)})}  = h_{\bm{a}_{\ell}}(\bm{\rho}_u) (= \bm{a}_{\ell}\cdot\bm{\rho}_u) \in \mathbb{R}$, 
$\ell=1,2,...,d'$.  
Subsequently, these hash values go through the same procedure to be combined into the AND hash codes as discussed in Sec.~\ref{sec:or-and}. This process is essentially equivalent to partitioning the input space into various distinct, non-overlapping regions randomly, 
and guarantees that even if the query and the key hash codes are processed separately,
only the attention between point pairs $u, v$ that are close in the geometric space is computed, which well addresses the misalignment issue.

\section{Related Work}\label{sec:relatedwork}
In this section, we review the most relevant work on efficient transformers and discuss their existing issues. More related work can be found in Appendix~\ref{app:related}.

\textbf{Neglecting Error-Computation Tradeoff.}
FLT~\cite{choromanski2023learning} models local inductive bias utilizing RFF for GDL tasks with tens of points and overlooks the bad error-computation tradeoff of RFF.
Those works using LSH, Reformer~\cite{kitaev2020reformer} and Smyrf~\cite{daras2020smyrf} consider OR-only LSH and neglect AND LSH, rendering non-neglectable error for large $n$; KDEformer~\cite{zandieh2023kdeformer} and HyperAttention~\cite{han2023hyperattention} employ AND-only LSH, which often does not work well in practice.
Furthermore, to employ angular-distance-based LSH functions, Reformer normalizes their queries and keys, limiting the model's expressiveness. KDEformer and HyperAttention avoid using normalized inputs, but using angular-distance-based LSH for maximum inner product search requires strong assumptions on the alignment between query-key angles and inner products.

\textbf{Compromised Accuracy for Computational Regularity.}
Reformer sorts and truncates hash buckets evenly for computational regularity, 
which does not guarantee that consecutive buckets in a hash table correspond to geometrically neighboring areas. To mitigate this issue, Smyrf, KDEformer, and HyperAttention utilize either E$^2$LSH~\cite{datar2004locality} or hyperplane LSH~\cite{charikar2002similarity}, ensuring that query/key pairs located in adjacent buckets are geometrically close.
However, these methods truncate queries and keys into equal-sized buckets separately, and neglect the query-key misalignment issue, as explained in Sec.~\ref{sec:recalibrate}. Both Flatformer~\cite{liu2023flatformer} and DSVT~\cite{wang2023dsvt} from CV propose to project 3D points onto the $x$ and $y$ axes rather than using LSH functions, and attention is computed by grouping points into equal-sized blocks along each axis. Such fixed projection directions may not be suitable for scientific problems with complex geometry. Moreover, some of these methods rely on domain-specific techniques, such as voxelization~\cite{wang2023dsvt}, which presents challenges in applications to general point-cloud data.

\section{Experiments}
We evaluate \proj for both predictive accuracy and computational performance against a variety of efficient transformers and GNNs on two critical tasks in HEP. 
In the following, we introduce our datasets, baselines, and experiment settings, and
more details can be found in Appendix~\ref{app:tasks} and~\ref{app:impl}.

\subsection{Datasets}
\textbf{Tracking-6k \& Tracking-60k.} We use two datasets derived from the TrackML Particle Tracking Dataset~\cite{amrouche2020tracking} designed for evaluating algorithms that reconstruct charged particle tracks, a crucial task in HEP that requires real-time processing.
During collision events, as charged particles pass through tracking detectors, they leave a trail of hits, each recorded with geometric coordinates and additional properties (e.g., momentum). The hits from a single collision event collectively form an attributed point cloud, as illustrated in Fig.~\ref{fig:tracking}.
The task is to identify which hits are left by the same particle and group them accordingly for track reconstruction. The current pipeline for this task is time-consuming, representing about 45\% of the total collider data reconstruction time~\cite{CMS:2815292}. Thus, accurate and efficient methods for this task are in great demand.
ML methods can be used to learn hit (point) embeddings such that the hits originating from the same particle are nearby in the embedding space for downstream clustering and track identification. 
Differing in scale, Tracking-6k comprises point clouds with about 6,000 points each, while Tracking-60k presents a more challenging scenario with each cloud containing about 60,000 points.

\textbf{Pileup-10k.} This dataset, similar to that in~\citet{martinez2019pileup,li2023semi}, is for the task of pileup mitigation, a critical data-denoising step in HEP. Each point cloud within the dataset represents an event resulting from multiple simultaneous proton-proton collisions at the LHC. The individual points in these clouds correspond to the particles generated from the collisions, either from the leading collision (LC) or simultaneous pileup collisions (PCs). The goal of this task is to classify whether each particle originates from the LC or PCs, as illustrated in Fig.~\ref{fig:pp}, which is a point classification task. There are 1000 point clouds in this dataset, each with about 10,000 points.

\subsection{Baselines and Setup}
\textbf{Efficient Transformer Baselines.} We evaluate eight efficient transformers from both NLP and CV domains as our baselines. These include the LSH-based Reformer~\cite{kitaev2020reformer}, Smyrf~\cite{daras2020smyrf}, and HyperAttn~\cite{han2023hyperattention}; RFF-based Performer~\cite{choromanski2020rethinking} and FLT~\cite{choromanski2023learning}; and ScatterBrain~\cite{chen2021scatterbrain}, which integrates both RFF and LSH approaches. From CV, we include Point Transformer~\cite{zhao2021point} and FlatFormer~\cite{liu2023flatformer}.

\textbf{GNN Baselines.} Besides collecting results from current SOTA GNNs for the two tasks~\cite{lieret2023high, li2023semi}, we use GCN~\cite{kipf2016semi} as a baseline and further benchmark three GNNs that have been widely used in scientific applications,
including GatedGNN~\cite{li2015gated, li2023semi}, DGCNN~\cite{DGCNN,qu2020jet}, and GravNet~\cite{qasim2019learning}.

\textbf{Random Baselines.} The random baselines for the Tracking datasets are implemented by randomly initializing \proj models without any training.
For the Pileup dataset, the random baseline is obtained by randomly assigning the output class probability for each point.

\textbf{Metrics.} 
For the tracking datasets, the quality of the learned point embeddings is assessed by evaluating how closely the embeddings of hits from the same particle cluster together. Specifically, we use the metric $\operatorname{AP@}k = \frac{1}{n} \sum_{u=1}^{n} \operatorname{Prec@}k_u$, where $k_u$ represents the number of hits originating from the same particle as hit $u$. $\operatorname{Prec@}k_u$ calculates the precision by retrieving the closest $k_u$ neighbors of hit $u$ in the embedding space. For the pileup dataset, the area under the precision-recall curve (AUC) is employed for this imbalanced binary classification task.

\textbf{Setup.} 
The transformer baselines are implemented using well-established codebases~\cite{fasttransformers, pytorchreformer, flyrepo} or author-provided code~\cite{liu2023flatformer}, while GNNs use the implementation from PyG~\cite{Fey_Fast_Graph_Representation_2019}.
For the results collected from SOTA GNNs, provided model checkpoints~\cite{lieret2023high} are used for evaluation on the Tracking datasets; for the Pileup dataset, the SOTA GNN is trained from scratch using available open-source code in~\cite{li2023semi}.
If not specified, point coordinates are used as the positional encoding for the transformer baselines following~\citet{liu2023flatformer, wang2023dsvt}.
All models are ensured to have a similar number of trainable parameters
and then the FLOPs used are aligned if possible. 
All models are trained and evaluated with the same seed to ensure reproducibility, using a server with NVIDIA Quadro RTX 6000 GPUs and Intel Xeon Gold 6248R CPUs. Note that all computations were performed on the GPUs, including the construction of k-nn graphs required by some baselines, where the API from PyG~\cite{Fey_Fast_Graph_Representation_2019} was used with k being 64 similar to previous works~\cite{lieret2023high, li2023semi}.

\textbf{Hyperparameter Tuning.}
The hyperparameters for the baselines and \proj are tuned with similar budgets, based on performance in the validation set of each dataset. For \proj, we adopt $m_1 = 3$ hash tables, each with $m_2 = 3$ hash functions for the three datasets. The block size of attention computation is set to $100$, and we use only point coordinates without point hidden representations as the AND hash inputs,
i.e., 
${L}_{\bm{{q}}_{u}}^{({i(1+\ell)})} = {L}_{\bm{{k}}_u}^{({i(1+\ell)})}  = h_{\bm{a}_{\ell}}(\bm{\rho}_u) (= \bm{a}_{\ell}\cdot\bm{\rho}_u)$, for $\ell=1,2$, where note that in HEP, the points are in a 2-d $\eta-\phi$ space~\cite{thais2022graph}, as detailed in Appendix~\ref{app:data}.
We set a fixed total number of buckets $\prod_{j=2}^{m_2} B_{ij}$ and generate different bucket sizes $\{B_{ij}\}$ randomly to mitigate the boundary effect.
See Appendix~\ref{app:impl} for detailed settings.

\begin{table}[t]
\centering
\caption{Predictive performance on the three datasets. The $\textbf{Bold}^\dagger$, $\textbf{Bold}^\ddagger$, and $\textbf{Bold}$ highlight the first, second, and third best results, respectively. $\underline{\text{Underline}}$ indicates the best transformer baselines.}
\vspace{-2mm}
\label{tab:main}
\resizebox{\columnwidth}{!}{%
\begin{tabular}{lccc}
\toprule
 & Tracking-6k (AP$@k$) & Tracking-60k (AP$@k$) & Pileup-10k (AUC) \\
\midrule
Random & 5.88 & 5.71 & 4.22 \\
SOTA GNNs & $\mathbf{91.00}^\ddagger$ & $\mathbf{90.89}^\ddagger$ & $\mathbf{40.26}$ \\
\midrule
Reformer & $72.37$ & $\underline{72.47}$ & $36.70$ \\
SMYRF & $72.98$ & $71.18$ & $25.20$ \\
HyperAttn & $71.49$ &	$70.22$	 & $25.31$ \\
Performer & $73.17$ & $72.07$ & $28.36$ \\
FLT & $72.55$ & $71.45$ & $25.26$ \\
ScatterBrain & $73.35$ & $72.06$ & $30.95$ \\
PointTrans & $72.33$ & $70.81$ & $\underline{40.26}$ \\
FlatFormer & $\underline{74.22}$ & $70.23$  & $38.61$ \\
\midrule
GCN & $79.61$ & $75.38$ & $40.10$ \\
DGCNN & $\mathbf{90.74}$ & $\mathbf{88.66}$  & $33.75$ \\
GravNet & $90.11$ & $87.99$ & $40.10$ \\
GatedGNN & $80.98$ & $78.42$ & $40.26$ \\
\midrule
Performer-$k_{\operatorname{\proj}}$ & $71.97$ & $69.20$ & $32.81$ \\
SMYRF-$k_{\operatorname{\proj}}$ & $83.19$ &	$71.04$ &	$\mathbf{40.31}^\ddagger$  \\
FlatFormer-$k_{\operatorname{\proj}}$ & $88.18$ & $85.06$ &	$39.99$ \\
\midrule
\proj & $\mathbf{92.66}^\dagger$ &	$\mathbf{91.93}^\dagger$ &	$\mathbf{40.39}^\dagger$ \\
\bottomrule
\end{tabular}
}
\vspace{-0.8cm}
\end{table}

\begin{table}[t]
\caption{Training and test time (ms) per sample. Each entry is the median from at least 100 measurements evaluated on an NVIDIA Quatro RTX 6000. Numbers in $(\cdot)$ are the time used to pre-construct input graphs that may be saved during training if pre-processing is allowed. Note that real-time inference requires building graphs on the fly. The $\textbf{Bold}^\dagger$ highlights the best results, and 
$\textbf{Bold}$ and $\underline{\text{Underline}}$ indicate the best transformer and GNN baselines, respectively.}
\label{tab:speed}
\vspace{-2mm}
\centering
\resizebox{\columnwidth}{!}{%
\begin{tabular}{lcccccc}
\toprule
\multicolumn{1}{c}{\multirow{2}{*}{}} & \multicolumn{2}{c}{Tracking-6k}          & \multicolumn{2}{c}{Tracking-60k}       & \multicolumn{2}{c}{Pilup-10k}
\\
\cmidrule(l{5pt}r{5pt}){2-3} \cmidrule(l{5pt}r{5pt}){4-5} \cmidrule(l{5pt}r{5pt}){6-7} 
\multicolumn{1}{c}{}                  
& Train & Test & Train & Test & Train & Test \\
\midrule
SOTA GNNs & $559$ & $221$ & OOM & $5781$ & $432(322)$ & $362$
\\
\midrule
Reformer         & 
                $355$ & $23.1$ & 
                $2570$ & $251$ & 
                $83.3$ & $23.4$
\\
SMYRF  & 
                $348$ & $8.7$ & $2343$ & 
                $69.6$ & $58.6$ & $12.4$
\\
HyperAttn &
$352$ &	$8.4$ & $2320$	&$62.1$&	$44.4$&	$12.5$

\\
Performer   & 
                $343$ & $\mathbf{8.3}$ & $2407$ &
                $68.7$ & $52.7$ & $12.8$
\\
FLT     & 
                $341$ & $8.4$ & $2369$ &
                $71.6$ & $55.9$ & $12.7$
\\
ScatterBrain     & 
                $357$ & $13.1$ & $2562$ &
                $129$ & $109$ & $34.6$
\\
PointTrans     & 
                $476 (130)$ & $144$ & $7361 (5017)$ &
                $5143$ & $372 (323)$ & $348$
\\
FlatFormer     & 
                $\mathbf{338}^{\dagger}$ & $\mathbf{8.3}$ & $\mathbf{2261}^\dagger$ &
                $\mathbf{58.7}$ & $\mathbf{53.7}$ & $\mathbf{12.2}$
\\
\midrule
GCN     & 
                $\underline{471 (129)}$ & $\underline{138}$ & $\underline{7332 (5009)}$ & $\underline{5123}$ & $376 (322)$ & $342$
\\
DGCNN     & 
                $563$ & $287$ & $14098$ &
                $11779$ & $325$ & $294$
\\
GravNet     & 
                $593$ & $251$ & $13597$ &
                $11684$ & $\underline{312}$ & $\underline{278}$
\\
GatedGNN     & 
                $512 (131)$ & $158$ & $7476 (5013)$ &
                $5263$ & $432 (328)$& $362$
\\
\midrule
\proj     & 
                $\mathbf{338}^\dagger$ & $\mathbf{7.0}^{\dagger}$ & $\mathbf{2312}$ &
                $\mathbf{57.9}^{\dagger}$ & $\mathbf{40.3}^{\dagger}$ & $\mathbf{10.7}^{\dagger}$
\\
\bottomrule
\end{tabular}%
}
\vspace{-3mm}
\end{table}

\begin{table}[t]
\centering
\caption{Ablation studies of \proj.}
\vspace{-0.2cm}
\label{tab:ablation}
\begin{small}
\begin{tabular}{lc}
\toprule
 &  {Tracking-60k} \\
\midrule
\proj w/o $k_{\operatorname{\proj}}$ & $72.28$  \\
\midrule
OR-only LSH & $71.42$ \\
OR-only LSH* & $78.22$  \\
OR \& AND LSH & $70.98$  \\
OR \& AND LSH* & $88.54$  \\
\bottomrule
\end{tabular}
\end{small}
\vspace{-0.5cm}
\end{table}

\begin{figure*}[ht]
    \centering
    \includegraphics[trim={0.0cm 0.0cm 0.0cm 0cm},clip,width=1.0\linewidth]{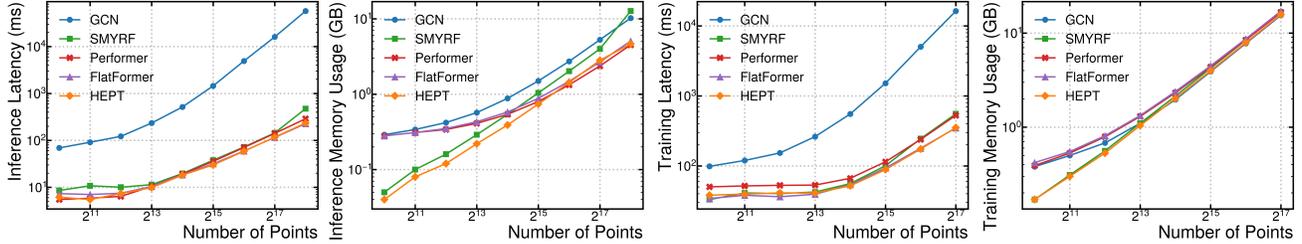}
    \vspace{-0.7cm}
    \caption{Inference and training costs per point cloud.
    }
    \label{fig:scaling}
    \vspace{-0.5cm}
\end{figure*}

\subsection{Result Analysis}
\textbf{Predictive Performance.} As shown in Table~\ref{tab:main}, GNNs are suitable for GDL tasks with local inductive bias and have indeed achieved good prediction accuracy. However, \proj still largely outperforms GNNs
(with much lower computational complexity).
When compared with other transformers, \proj's performance gain is even more significant (up to $22\%$).
To inspect the benefits of our proposed kernel $k_{\operatorname{\proj}}$, we also incorporate it into some transformer baselines when possible. These baselines also yield substantial improvements, validating the necessity of modeling local inductive bias explicitly for GDL tasks in HEP. 
Moreover, we observe that RFF-based methods Performer and FLT consistently exhibit unsatisfactory performance even with $k_{\operatorname{\proj}}$, which aligns with our analysis. As for LSH-based methods, 
SMYRF shows promise with $k_{\operatorname{\proj}}$, but it is unable to well generalize to the larger dataset Tracking-60k due to its OR-only LSH-based design and the neglect of query-key alignment. Similarly, FlatFormer also achieves good results when paired with our kernel. However,
it still falls short of matching the SOTA GNNs and \proj.

\textbf{Computational Complexity.} Table~\ref{tab:speed} compares both the training (forward + backward pass) and inference (forward pass) time per sample for all models yielded from Table~\ref{tab:main}, and their FLOPs and GPU memory usage are reported in Table~\ref{tab:memory} in the appendix.
Clearly, \proj is among the most efficient transformers, and
the gain in computational speed compared with GNNs is tremendous, especially for large point clouds.  
The speedup can be less significant in training for the Tracking datasets since the loss computation in this task dominates the running time in training (see Table~\ref{tab:training_latency_bd}).
Specifically, in Tracking-60k, \proj achieves an $89$-$203\times$ speedup in inference and a $3$-$6\times$ speedup in training, and in Pileup-10k, \proj is $26$-$34\times$ faster for inference and $8$-$11\times$ faster for training, all while maintaining SOTA predictive accuracy.
The slower performance of GNNs is due to the need for constructing graphs from point clouds and their irregular computation, while efficient transformers such as \proj avoid graph construction and adopt only efficient regular computation.
Moreover, \proj can be further accelerated by applying such as $\operatorname{Float16}$ computation and FlashAttention~\cite{dao2023flashattention}, which we leave as future studies.

\textbf{Ablation Studies.} We conduct ablation studies whose results are shown in Table~\ref{tab:ablation}.
\emph{First}, we evaluate the importance of our proposed kernel, where we replace our kernel with the kernel from the original transformer with absolute positional encoding. However, using the traditional kernel significantly reduces the performance of the model.
\emph{Second}, to show the effectiveness of OR \& AND LSH in a general setting, we remove the use of point coordinates when obtaining aux hash codes. These codes are now computed only based on the query/key vectors. Since query-key alignment is critical, without using point coordinates, we follow~\cite{kitaev2020reformer} by tying query and key vectors, i.e., $\bm{q}_u = \bm{k}_u$ for alignment.
The models with such query-key alignment are highlighted with $^*$. As Table~\ref{tab:ablation} shows, query-key alignment is important. With query-key alignment, the advantage of OR \& AND LSH over OR-only LSH is obvious. \proj by using OR \& AND LSH and point coordinates for query-key alignment achieves the best performance.

\subsection{Scalability Analysis}
The considered tasks cover point clouds with 6k, 10k, and 60k points, offering a preliminary view of HEPT's scalability. To further examine scalability across a broader range of input sizes, we evaluate HEPT on point clouds ranging from 1k ($2^{10}$) to 262k ($2^{18}$) points. 
The results are presented in Fig.~\ref{fig:scaling}, where we also include a comparison with GCN and the three most efficient transformer baselines as indicated by Table~\ref{tab:speed}, and the same settings used in Table~\ref{tab:speed} and Table~\ref{tab:memory} for the pileup mitigation task are employed. The input point clouds are generated randomly to meet the required number of points, and all models are closely matched in terms of FLOPs and trainable parameters (see Table~\ref{tab:memory}). Fig.~\ref{fig:scaling} indicates that HEPT is among the most scalable efficient transformers in terms of both latency and memory usage, even with input sizes extending from $2^{10}$ to $2^{18}$.

\subsection{Sensitivity Analysis}

\begin{table}[t]
\centering
\caption{Performance of HEPT on Tracking-60k with different configurations of hash tables and bucket sizes. Results are reported as $\operatorname{AP@}k\,(\text{GFLOPs})$, i.e., the numbers in parentheses represent the GFLOPs for each configuration.}
\vspace{-2mm}
\label{tab:sensitivity}
\resizebox{\columnwidth}{!}{%
\begin{tabular}{cccc}
\toprule
\multicolumn{1}{c}{\multirow{2}{*}{\# Hash Tables}} & \multicolumn{3}{c}{Block Size} \\
\cmidrule(l{5pt}r{5pt}){2-4}
& 50 & 100 & 150 \\
\midrule
1 & $73.57\,(24.2)$ & $78.60\,(29.8)$ & $80.91\,(35.4)$ \\
3 & $87.47\,(35.6)$ & $91.93\,(52.2)$ & $92.22\,(68.9)$ \\
5 & $91.89\,(46.9)$ & $92.27\,(74.7)$ & $92.34\,(102.6)$ \\
\bottomrule
\end{tabular}
}
\vspace{-6mm}
\end{table}

Table~\ref{tab:sensitivity} evaluates HEPT on Tracking-60k by varying the number of hash tables and block sizes, keeping other settings the same as those used in Table~\ref{tab:main}. 
Notably, configurations using a single hash table correspond to AND-only LSH, which generally performs poorly in practice and our results further verify this claim. For other configurations, HEPT demonstrates robustness, with increased computational budgets generally improving performance.

\vspace{-1mm}
\section{Conclusion}
This work introduces \proj, a new efficient transformer architecture for fast and accurate large-scale point cloud learning in scientific domains. Quantitative analysis on error-computation tradeoff shows the inherent limitations of RFF and the necessity of using OR \& AND LSH to design efficient transformers for applications with local inductive bias. 
Two tasks in HEP have been used for evaluation, where \proj greatly boosts computational speed and predictive accuracy against existing GNNs and transformers.

\section*{Acknowledgement}
The authors thank Kilian Lieret, Gage DeZoort, and Yongbin Feng for
their helpful discussions.
S. Miao, M. Liu, and P. Li are partially supported by the National Science Foundation (NSF) award PHY-2117997 and IIS-2239565. J. Duarte is also supported by the NSF award PHY-2117997 and by the Department of Energy (DOE) award DE-SC0021187.

\section*{Impact Statement}
This paper presents work whose goal is to advance the field of 
Machine Learning. There are many potential societal consequences 
of our work, none which we feel must be specifically highlighted here.


\bibliography{reference.bib}
\bibliographystyle{icml2024}

\newpage
\appendix
\onecolumn

\begin{figure}[t]
    \centering
    \begin{subfigure}[t]{0.32\columnwidth} 
        \centering
        \includegraphics[trim={0cm 0cm 0.2cm 0cm}, clip, width=1.0\linewidth]{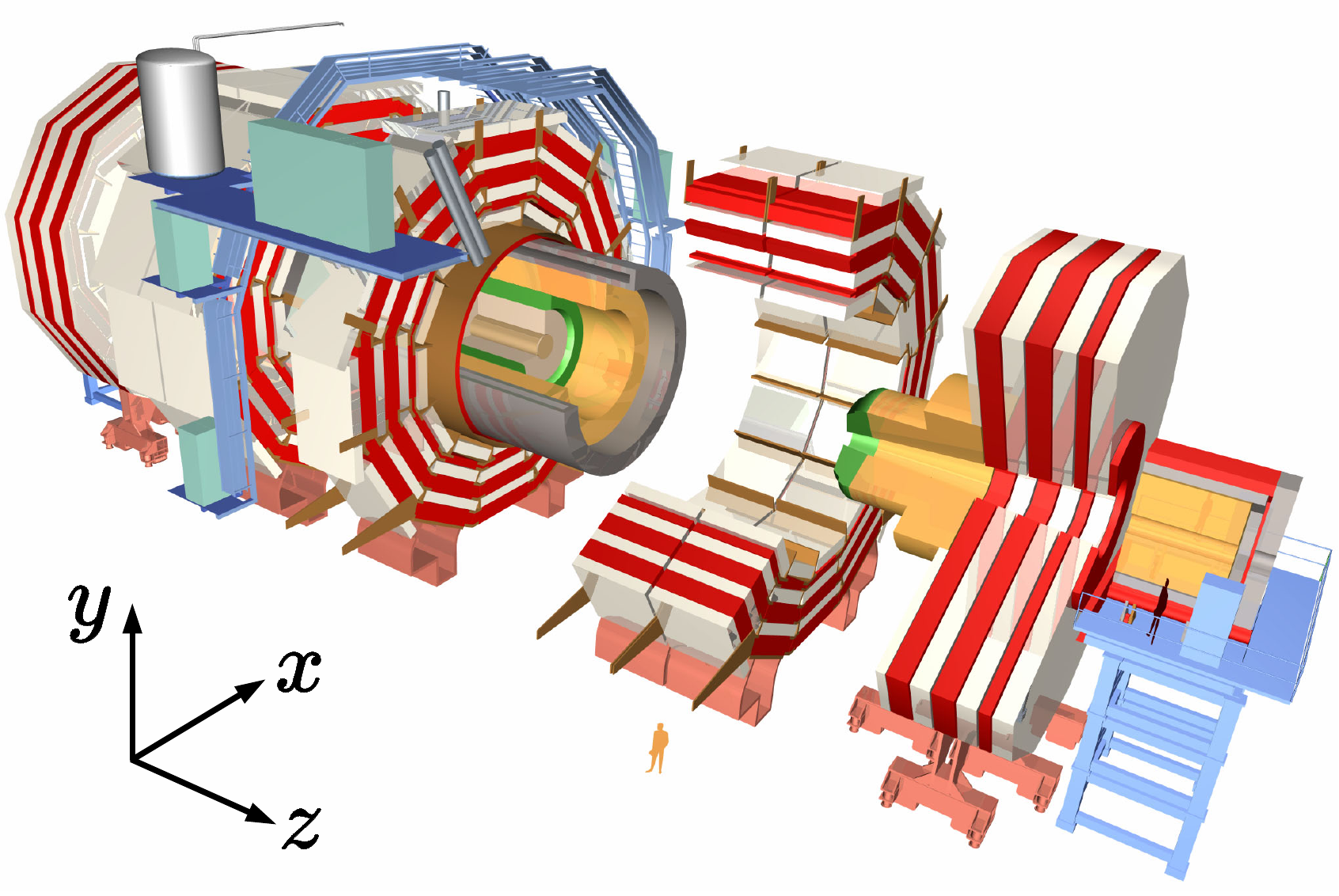}
        \caption{3D view of CMS detector}
    \end{subfigure}
    \hspace{0.01\columnwidth}
    \begin{subfigure}[t]{0.32\columnwidth}
        \centering
        \includegraphics[trim={1.3cm 0.0cm 0cm 0cm}, clip,width=0.70\linewidth]{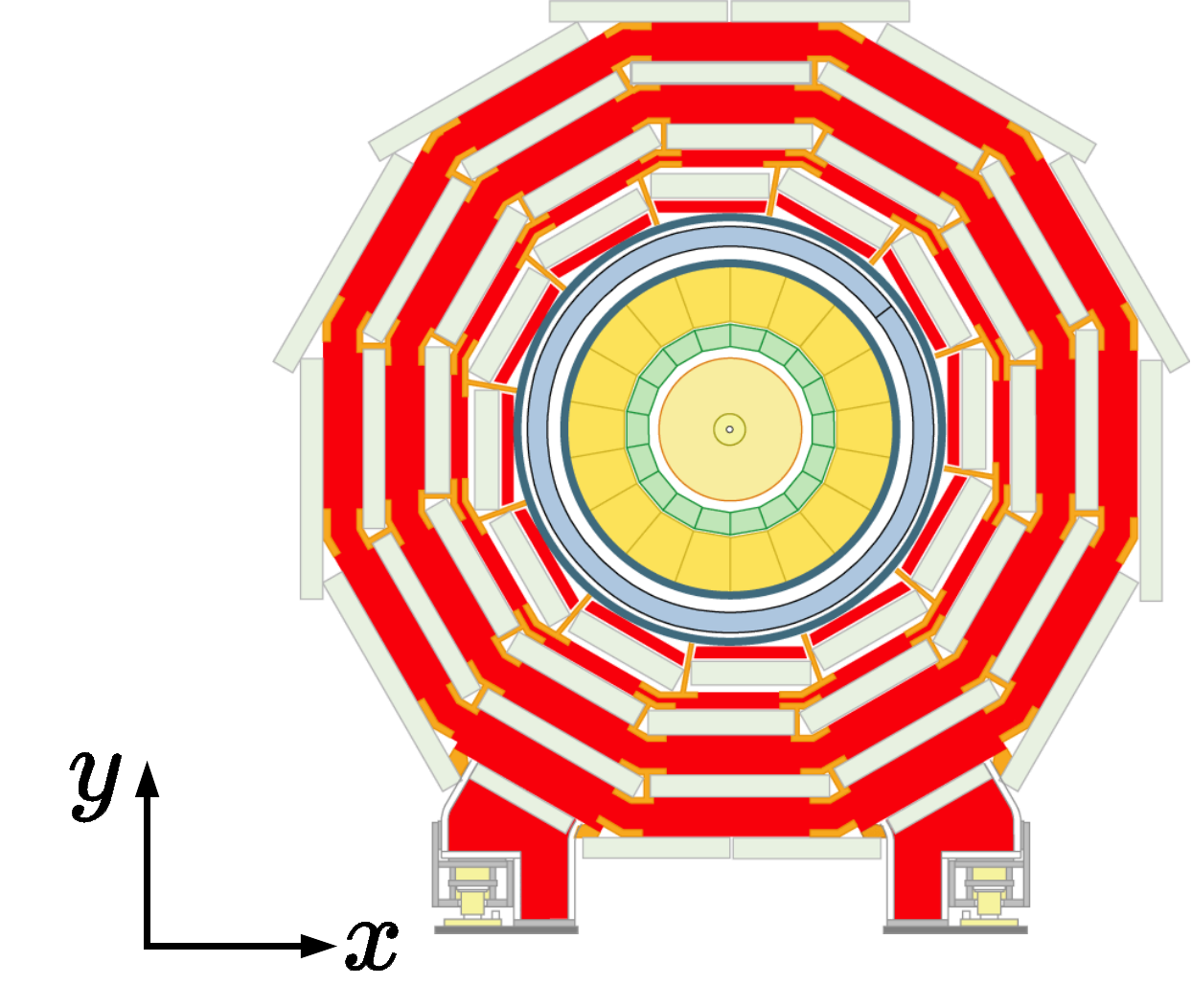}
        \caption{Transverse view of CMS detector}
    \end{subfigure}
    \hspace{0.01\columnwidth}
    \begin{subfigure}[t]{0.32\columnwidth}
        \centering
        \includegraphics[trim={0cm -0.5cm 0.1cm 0.0cm}, clip,width=1\linewidth]{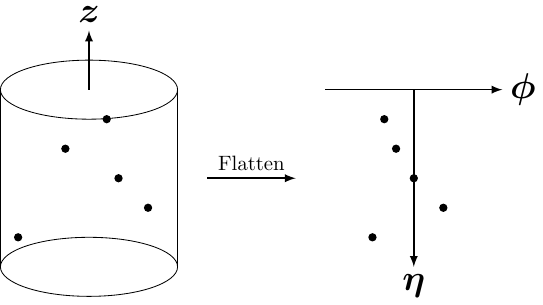}
        \caption{Converting to $\eta-\phi$ space}
        \label{fig:convert}
    \end{subfigure}
    \vspace*{-2mm}
    \caption{Visualizations of CMS detector and coordinate systems in HEP analysis, adapted from \citet{CMS-PHO-GEN-2012-002}.}
    \label{fig:coords}
\end{figure}

\begin{table}[t]
\caption{Statistics of the three datasets.}
\label{tab:datasets}
\vspace{-2mm}
\centering
\resizebox{\columnwidth}{!}{%
\begin{tabular}{lcccccc}
\toprule
& \# Point Clouds & \# Features in $\bm{X}$&  \# Dimensions in $\bm{\rho}$  &  Avg. \# Points per Cloud & Avg. \# Labeled Pairs per Cloud 
& Class Ratio (Pos./Neg.)\\
\midrule
{Tracking-6k}  & 500  &16  & 2  &6.8k & 3.7M &N/A    \\
{Tracking-60k} & 50   &16  & 2  &56.7k& 75.5M&N/A  \\
{Pileup-10k} & 1000  & 8 & 2 &10.3k & N/A  & 0.039 \\
\bottomrule
\end{tabular}%
}
\end{table}

\begin{figure}[t]
    \centering
    \includegraphics[trim={0cm 0.0cm 0.0cm 0cm}, clip, width=0.9\linewidth]{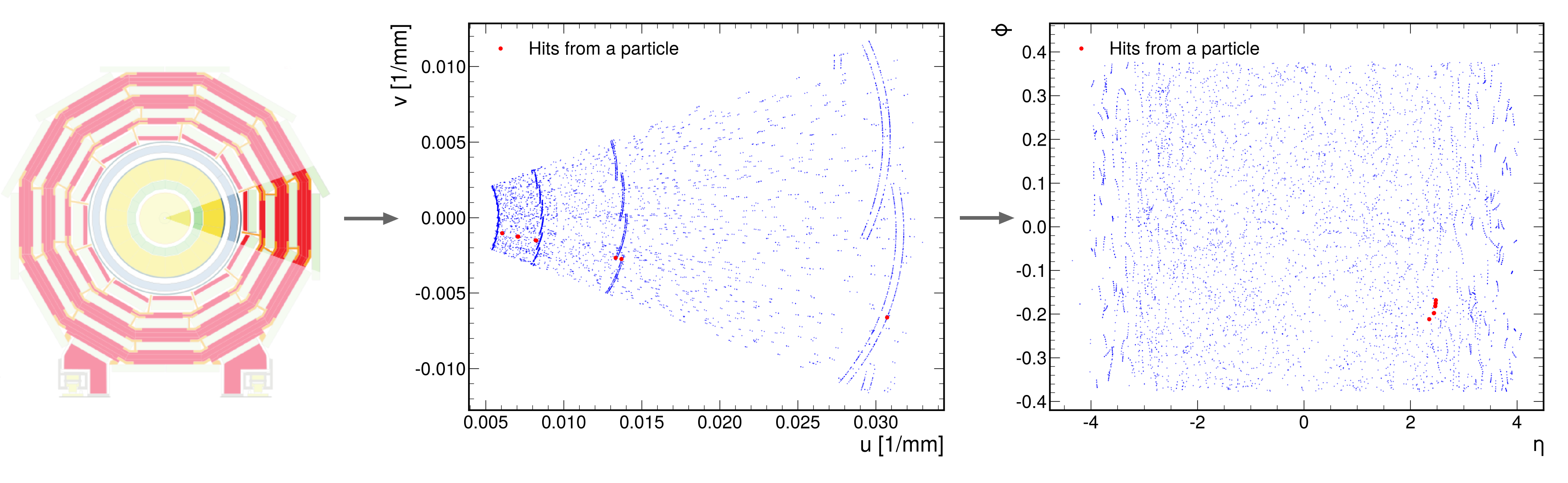}
    \caption{Visualization of a sample from the Tracking datasets, showcasing only the points collected from the detectors in the highlighted region for better illustration. The leftmost part of this figure is adapted from~\citet{CMS-PHO-GEN-2012-002}.}
    \label{fig:tracking_data}
\end{figure}

\begin{figure}[t]
    \centering
    \includegraphics[trim={0cm 0cm 0.0cm 0cm}, clip, width=0.7\linewidth]{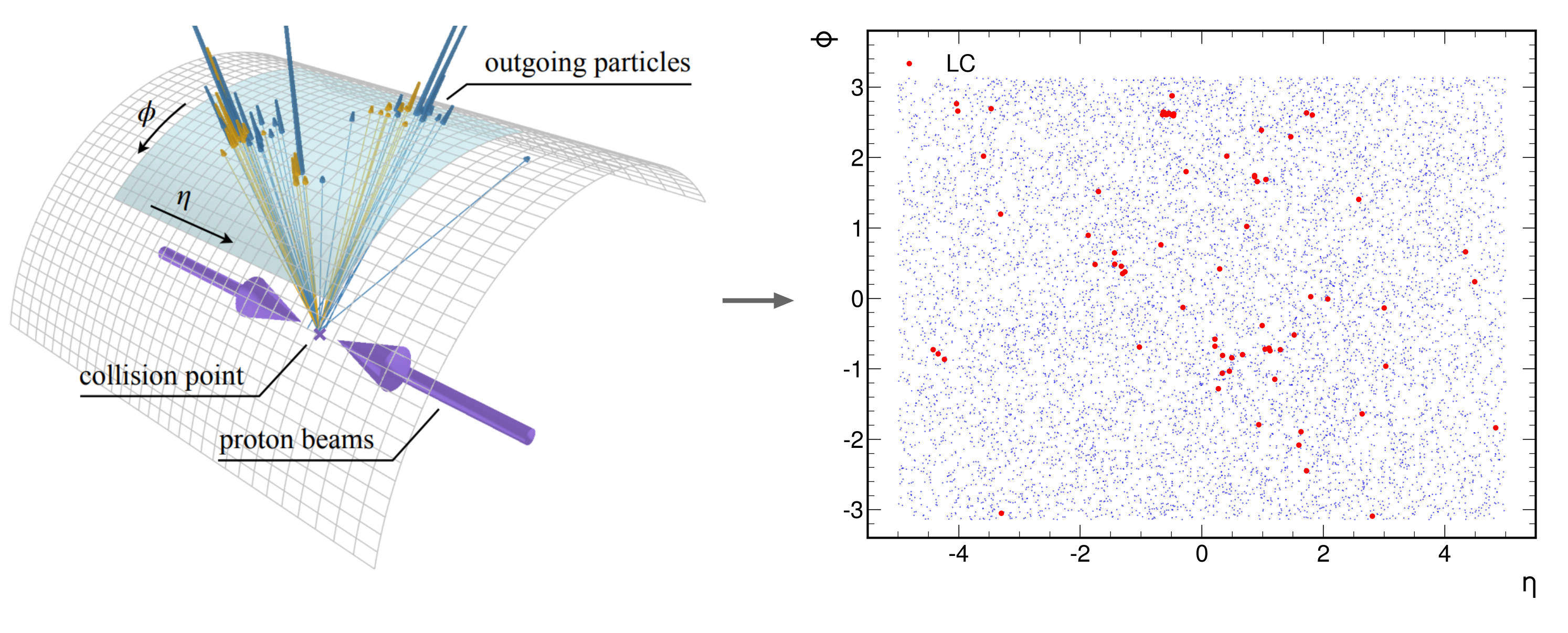}
    \caption{Visualization of a sample from the Pileup dataset. The left part of this figure is adapted from~\citet{qu2022particle}.}
    \label{fig:pp_data}
\end{figure}

\section{Details of the Datasets}\label{app:tasks}

\subsection{Background}
\textbf{Tracking-6k \& Tracking-60k.} 
These two datasets are for charged particle tracking at the CERN LHC, which 
is a crucial task in HEP as it enables precise identification and reconstruction of charged particles' paths, facilitating the determination of their momentum, energy, etc. This capability is crucial for identifying particle types and reconstructing collision events, laying the foundation for precise measurements of particle properties such as mass and charge. These measurements are vital for testing Standard Model predictions and probing for new physics~\cite{langacker2017standard}. Additionally, accurate tracking is instrumental in suppressing background noise, distinguishing between signal events and the more common processes, thereby enhancing the detection of rare phenomena and contributing significantly to our understanding of fundamental particles and their interactions at high energies.
The tracking process utilizes sophisticated detector systems, such as the inner detector of the ATLAS and CMS experiments, to reconstruct particle trajectories from collision events. 
However, challenges arise from the vast volume of data generated, background noise, and experimental complexities, necessitating robust yet efficient algorithms, e.g.,
the LHC operates at an extremely high collision rate, with millions of proton-proton collisions occurring every second to be analyzed.
Traditional combinatorial-Kalman-filter-based track reconstruction~\cite{Strandlie:2010zz} cannot easily scale up to future LHC data and is difficult to parallelize on heterogeneous computing platforms. And the inherent complexity of GNNs renders it hard for their efficient deployment at the LHC.
This study is performed using the TrackML dataset~\cite{amrouche2021tracking}, which simulates the worst-case future LHC pileup conditions (200 interactions per proton bunch crossing) in a generic tracking detector geometry.

\textbf{Pileup-10k.} 
This dataset focuses on pileup mitigation, a critical challenge in analyzing data from the LHC, where multiple proton-proton collisions occur simultaneously within the same or nearby bunch crossings. These overlapping interactions, known as pileup collisions (PCs), complicate the extraction of meaningful data from the primary collision of interest. Effective pileup mitigation is essential for maintaining the physics sensitivity of LHC experiments, as it involves distinguishing and removing the contributions of noisy particles from PCs to isolate signals from the leading collision (LC), which is associated with the primary vertex having the highest sum of particle momentum. During the 2016 to 2018 LHC runs, the average pileup level was around 40, but this figure is anticipated to rise to as much as 200 in future runs (i.e., more noise and larger input sizes), significantly increasing the complexity of data analysis. The reconstruction of particles from LHC collisions relies on tracking detector hits and calorimeter energy deposits. While the SOTA tracking systems allow for tracking and vertexing of charged particles, enabling the straightforward identification and removal of those associated with PCs, the main challenge lies in dealing with neutral particles, such as photons and neutral hadrons, which do not leave tracks. To address these challenges, simulation samples based on the {\tt DELPHES} framework~\cite{deFavereau:2013fsa} are utilized, generating both charged and neutral particles from selected physics processes alongside detector resolution effects, to develop and test pileup mitigation strategies.

\begin{table}[t]
\caption{Point Features in the two tasks.}
\label{tab:point_features}
\vspace{-2mm}
\centering
\resizebox{0.9\textwidth}{!}{%
\begin{tabular}{lll}
\toprule
Task & Variable & Definition \\
\midrule
\multirow{14}{*}{Tracking} & $r$  & Radial distance from the beam axis in cylindrical coordinates. \\
 & $\phi$ & Azimuthal angle around the beam axis in cylindrical coordinates. \\
 & $z$ & Longitudinal position along the beam axis. \\
 & $\eta$ & Pseudorapidity, measuring the angle of particle trajectory relative to the beam axis. \\
 & $u$ & Local coordinate axis in a detector plane, orthogonal to $v$. \\
 & $v$ & Local coordinate axis in a detector plane, orthogonal to $u$. \\
 & $\operatorname{charge\_frac}$ & Fraction of the charge collected by a sensor, indicating the quality of a hit. \\
 & $\ell_{\eta}$ & Local pseudorapidity, calculated within a specific sub-detector region. \\
 & $\ell_{\phi}$ & Local azimuthal angle, measured within a specific sub-detector region. \\
 & $\ell_{x}$ & Local $x$ coordinate, representing position within a sub-detector. \\
 & $\ell_{y}$ & Local $y$ coordinate, representing position within a sub-detector. \\
 & $\ell_{z}$ & Local $z$ coordinate, representing position along the beam axis within a sub-detector. \\
 & $g_{\eta}$ & Global pseudorapidity, calculated with respect to the overall detector geometry. \\
 & $g_{\phi}$ & Global azimuthal angle, measured with respect to the overall detector geometry. \\
 \midrule
\multirow{8}{*}{Pileup} & $\eta$  & Pseudorapidity, a measure of the angle relative to the beam axis. \\
& $\phi$ & Azimuthal angle around the beam axis in cylindrical coordinates. \\
& $p_x$ & Momentum component of the particle in the $x$ direction. \\
 & $p_y$ & Momentum component of the particle in the $y$ direction. \\
 & $p_t$ & Transverse momentum, calculated from the $x$ and $y$ momentum components. \\
 & $\operatorname{Rapidity}$ & A measure of the particle's velocity in the direction of the beam. \\
 & $E$ & Energy of the particle. \\
 & $\operatorname{PID}$ & Particle ID, indicating the type of the particle, e.g., muon, electron, etc. \\
\bottomrule
\end{tabular}%
}
\end{table}

\subsection{Data Format}\label{app:data}
The statistics of the three datasets used are shown in Table~\ref{tab:datasets}, and
Fig.~\ref{fig:tracking_data} and Fig.~\ref{fig:pp_data} provide visualizations of the data samples.
Below we introduce the features and labels included in the datasets.

\textbf{Point Coordinates $\bm{\rho}$.} Following the standard pipeline in HEP analysis, each point in the three datasets is associated with a 2-d coordinate in $\eta-\phi$ space~\cite{thais2022graph}. Fig.~\ref{fig:coords} visualizes the CMS detector at the LHC with its 3D and transverse view, where collisions occur at the center, surrounded by multiple layers of cylindrical detectors, and particles flying out from the center will hit the detectors at various locations. Imagine cutting the cylindrical detector along its length and flattening it out into a 2D plane, as illustrated in Fig.~\ref{fig:convert}. The vertical axis of this plane can represent the pseudorapidity ($\eta$), which indicates how far up or down (along the beam axis) the particle hit the detector. The horizontal axis represents the azimuthal angle ($\phi$), indicating the particle's direction around the beam axis.

\textbf{Point Features $\bm{X}$.} Table~\ref{tab:point_features} lists the variables included as point features in the two tasks and their definitions. Note that geometric features, e.g., $\eta$ and $\phi$, are also included as point features for model learning, which is a common practice in HEP and results in non-equivariant models with respect to these geometric features. We follow this practice in our implementation similar to previous works~\cite{lieret2023high, li2023semi} and whether equivariant models are useful in HEP has not reached a consistent conclusion~\cite{thais2023equivariance}.

\textbf{Ground-Truth Labels.} For the Tracking datasets, any pairs of hits (points) from the same particle are labeled as positive samples to be learned with similar embeddings, while for each hit its neighboring 256 hits from other particles are labeled as negative pairs. For the Pileup dataset, a particle (point) is labeled positive if it is from LC, and otherwise, it is labeled negative. Note that this task is highly imbalanced, and only about $3.9\%$ of points are labeled positive.

\subsection{Task Formulation}\label{app:taskFormulation}
Below we describe how they are formulated as ML tasks.

\textbf{Tracking-6k \& Tracking-60k.} To learn clustered embeddings for hits originating from the same particle, we adopt contrastive learning with InfoNCE loss~\cite{oord2018representation}. 
For any pairs of hits from the same particle, they are labeled as positive samples to be learned with similar embeddings, while for each hit its neighboring 256 hits from other particles are labeled as negative pairs. 
Therefore, with learned embeddings for a point $u$, denoted as $\bm{h}_u$,
the loss is computed as 
$$
    \mathcal{L}_{\operatorname{InfoNCE}} = - \log \frac{\exp(\operatorname{sim}(\bm{h}_u, \bm{h}_v^+))}{\exp(\operatorname{sim}(\bm{h}_u, \bm{h}_v^+)) + \sum_{\bm{h}_v^- \in \mathcal{N}} \exp(\operatorname{sim}(\bm{h}_u, \bm{h}_v^-))},
$$
where $\operatorname{sim}$ is some similarity metric, $\bm{h}_v^+$ is the embeddings of a point $v$ that is a positive pair with point $u$, and $\mathcal{N}$ is a set of negative pairs for point $u$, whose embeddings are denoted as $\bm{h}_v^-$. 
In our experiments, we adopt $\operatorname{sim}$ as $\exp(- \| \bm{h}_u - \bm{h}_v \|^2/\tau)$, where $\tau$ is a positive hyperparamter. We also experimented with angular distances and dot product as the similarity metric, and they can hardly work even with GNNs that need no approximation in computation.
The dataset split is done in a cloud-wise way, i.e., $80\%$/$10\%$/$10\%$ of point clouds are used to train/validate/test models, respectively.

\textbf{Pileup-10k.} Each point cloud consists of both charged and neutral particles (points), and models are only trained to predict the class of each neutral particle, i.e., from LC or PCs. Since this is an imbalanced binary classification task, meaning more points are labeled from PCs, the Focal loss~\cite{lin2017focal} is adopted, i.e., a variant of cross-entropy loss for imbalanced classification, for better performance. The dataset split is also done in a cloud-wise way, i.e., $80\%$/$10\%$/$10\%$ of point clouds are used to train/validate/test models, respectively.

\section{Extended Related Work}\label{app:related}
\textbf{Efficient Transformers in NLP.}
Leveraging local inductive biases in NLP, several works introduce local attention patterns~\cite{child2019generating, tay2020sparse, beltagy2020longformer, zaheer2020big}.
Other approaches focus on exploiting the inherent properties of the attention matrix, employing techniques such as RFF or Nystrom for low-rank approximations~\cite{wang2020linformer, xiong2021nystromformer, peng2021random, choromanski2020rethinking}, leveraging the sparsity of the attention matrix with various LSH-based methods~\cite{kitaev2020reformer, daras2020smyrf, zandieh2023kdeformer, han2023hyperattention}, or combining both properties~\cite{chen2021scatterbrain}. Research has also been conducted on optimizing attention computation on hardware like GPUs~\cite{dao2022flashattention, dao2023flashattention} and in designing new linear transformers that replace the original Softmax attention~\cite{katharopoulos2020transformers, sun2023retentive}.

\textbf{Other Efficient Transformers.}
In CV, domain-specific knowledge has led to the adoption of local neighborhood attention~\cite{zhao2021point, mao2021voxel} and the techniques that partition 2D spaces into grids for parallel processing~\cite{fan2022embracing, sun2022swformer}. To further improve computational regularity, some methods~\cite{liu2023flatformer, wang2023dsvt} project points/voxels onto axes, forming equal-sized point sets along each axis.
In scalable graph transformers, where the input graphs are assumed to be given, approaches vary from sampling-based techniques for large graphs~\cite{chen2022nagphormer, zhang2022hierarchical, wu2022nodeformer} to those that approximate the spectral properties of input graphs~\cite{shirzad2023exphormer}.

\textbf{Attention Kernels \& Relevant Positional Encoding.} The original transformer~\cite{vaswani2017attention} utilizes the attention kernel $\exp(\bm{q}_i^\top\bm{k}_j)$ with absolute positional encoding (PE) added to the query/key vectors to capture positional information between tokens in sentences.
Following this idea, many works, such as those for 3D detection from CV~\cite{liu2023flatformer, wang2023dsvt} adopt PE similar to this. Instead of absolute PE, there are also studies utilizing relative positional encoding (RPE), and one way to formulate it is $\exp(\bm{q}_i^\top\bm{k}_j +  b_{j-i})$, where $b_{j-i}$ could be some learnable parameters for each relative position~\cite{shaw2018self, wu2021rethinking}. 
Recently, FLT~\cite{choromanski2023learning} also considers generalizing the idea of RPE to GDL tasks with point coordinates, and they adopt RFF-based methods with a kernel, e.g., 
$\exp(\bm{q}_i^\top\bm{k}_j + \omega \exp (- \| \bm{\rho}_i - \bm{\rho}_j \|^2 / \sigma^2))$,  where $\bm{\rho}$ is point coordinates. However, as demonstrated by our analysis and empirically in Table~\ref{tab:main}, such RFF-based RPE implementation does not work well for large-scale point clouds with local inductive bias (and there may not be an easy way to adopt LSH-based methods to approximate this kernel). Our proposed kernel can also be viewed as a type of RPE regarding the way to leverage the point coordinates, as it effectively incorporates distance information between points in attention computation. However, no efficient transformers (with near-linear complexity) have been developed to enable the approximation of our type of RPE. HEPT is the first work that enables effective approximations of it via LSH and makes the obtained efficient transformer better suited for GDL tasks with local inductive bias.

\textbf{Point Cloud Serialization for Efficient Point Cloud Processing.} 
Recently, Point Transformer V3~\cite{wu2024point} summarizes a series of works in CV as point cloud serialization (PCS) techniques, which project irregular point cloud data into regular sequences with locality in the original 3D space being preserved to some degree for efficient computation~\cite{liu2023flatformer, wang2023dsvt, Wang2023OctFormer}.
These studies typically employ fixed serialization patterns, such as those induced from Hilbert or Z-order curves~\cite{wu2024point}. 
Actually, HEPT can also be seen as a PCS technique, but it uses \emph{randomized} serialization patterns through LSH to project point clouds into regular sequences for computing block-diagonal attention. 
Since those fixed serialization patterns would overlook specific locality patterns~\cite{wu2024point} and only work for low-dimensional data, the \emph{randomized} approach in HEPT offers an effective alternative or complementary PCS method, due to its guaranteed and analyzable ability to capture locality in the data, even when it is high-dimensional.

\section{Implementation Details}\label{app:impl}

\subsection{Implementation of \proj}
We follow the standard architecture of the transformer~\cite{vaswani2017attention}, i.e.,
\begin{align*}
    \bm{H} &= \operatorname{LN}(\bm{H}^{\ell}) \\
    \bm{Q}=[\bm{H} \bm{W}_Q \|  \sqrt{2\omega} \bm{\rho} ], \quad \bm{K}&=[\bm{H} \bm{W}_K \| \sqrt{2\omega} \bm{\rho} ], \quad \bm{V}=\bm{H} \bm{W}_V \\
    \bm{H}^{\prime} & = \bm{H} + \operatorname{MHSA}(\bm{Q}, \bm{K}, \bm{V}) \\
    \bm{H}^{\ell+1} & = \bm{H}^{\prime} + \operatorname{FFN}(\operatorname{LN}(\bm{H}^{\prime})),
\end{align*}
where $\bm{H}^{\ell} \in \mathbb{R}^{n\times h}$ is the learned point embeddings at the $\ell^{th}$ layer, $\operatorname{LN}$ is the layer normalization~\cite{ba2016layer}, $\bm{W}_Q, \bm{W}_K, \bm{W}_V$ are learnable projection matrices, $\bm{\rho} \in \mathbb{R}^{n\times k_2}$ is point coordinate matrix, $\omega\in \mathbb{R}^{+}$ are positive learnable parameters,
$\|$ concatenates two matrices along the column dimension, and $\operatorname{FFN}$ denotes a feed-forward layer. $\operatorname{MHSA}$ is the multi-head self-attention mechanism, and in our case, the unnormalized attention scores between query-key pairs will be computed via our kernel, and the full attention matrix is approximated by our LSH-based methods illustrated in Fig.~\ref{fig:arch}.

In our implementation for the three datasets, \proj uses $4$ layers and $24$ hidden dimensions with $8$ attention heads in each layer. In addition, we adopt $m_1 = 3$ hash tables, each with $m_2 = 3$ hash functions for the three datasets. The block size of attention computation is set to $100$, and we use only point coordinates without point hidden representations as the AND hash inputs, 
i.e., 
${L}_{\bm{{q}}_{u}}^{({i(1+\ell)})} = {L}_{\bm{{k}}_u}^{({i(1+\ell)})}  = h_{\bm{a}_{\ell}}(\bm{\rho}_u) (= \bm{a}_{\ell}\cdot\bm{\rho}_u)$, for $\ell=1,2$, where in HEP, the points are in a 2-d $\eta-\phi$ space~\cite{thais2022graph}, as detailed in Appendix~\ref{app:data}.
The total number of buckets $\prod_{j=2}^{m_2} B_{ij}$ is tuned for different datasets. 

For the Tracking datasets, the resulting 24-dimensional point embeddings are first projected into 12 dimensions to ensure a fair comparison with SOTA GNNs, which output 12-dimensional final point embeddings. Then, the embeddings are fed into the InfoNCE loss as described in Sec.~\ref{app:taskFormulation} to learn similar point embeddings for points from the same particle and dissimilar embeddings for those from different particles.
For the Pileup dataset, the resulting point embeddings are projected into 1 dimension with a $\operatorname{sigmoid}$ layer for the computation of Focal loss.

\begin{table}[t]
\caption{FLOPs (G) and GPU memory usage (GB). The $\textbf{Bold}^\dagger$, $\textbf{Bold}^\ddagger$, and $\textbf{Bold}$ highlight the first, second, and third best results, respectively. Note that SOTA GNNs for the Tracking datasets employ rather large models and involve complex operations along edges, leading to a significant amount of FLOPs. Besides these two models, other models for the same dataset are ensured to have the same number of trainable parameters and similar FLOPs if possible. GNNs, despite with fewer FLOPs, may have high training/test time due to (dynamic) graph construction and irregular computations.}
\label{tab:memory}
\vspace{-2mm}
\centering
\resizebox{\columnwidth}{!}{%
\begin{tabular}{lccccccccc}
\toprule
\multicolumn{1}{c}{\multirow{2}{*}{}} & \multicolumn{3}{c}{Tracking-6k}          & \multicolumn{3}{c}{Tracking-60k}       & \multicolumn{3}{c}{Pilup-10k}
\\
\cmidrule(l{5pt}r{5pt}){2-4} \cmidrule(l{5pt}r{5pt}){5-7} \cmidrule(l{5pt}r{5pt}){8-10} 
\multicolumn{1}{c}{}                  
& FLOPs & Train Mem. & Test Mem.& FLOPs & Train Mem. & Test Mem. & FLOPs & Train Mem.& Test Mem.\\
\midrule
SOTA GNNs & $316.1$ & $4.7$ & $1.8$ & $3717.8$ & OOM & $16.9$ & $3.9$ & $3.3$ & $1.5$
\\
\midrule
Reformer         & 
                $6.8$ & $1.8$ & 
                $0.7$ & $101.0$ & 
                $19.0$ & $9.4$ & $8.0$ & $\mathbf{1.4}$ & $0.6$
\\
SMYRF  & 
                $5.3$ & $1.4$ & $\mathbf{0.6}$ &  
                $54.6$ & $20.9$ & $3.6$ & 
                $8.1$ & $1.6$ & $\mathbf{0.5}$
\\
HyperAttn &
$5.2$ & $\mathbf{1.2}^{\ddagger}$ &$0.52$ &
$54.3$ & $19.4$ &	$3.3$ &
$8.7$& $1.5$ &$0.6$
\\
Performer   & 
                $5.6$ & $\mathbf{1.3}$ & $\mathbf{0.6}$ &
                $58.6$ & $20.0$ & $3.7$ & $9.7$ & ${1.5}$ & $0.6$
\\
FLT     & 
                $6.4$ & $\mathbf{1.2}^{\ddagger}$ & $\mathbf{0.5}^{\ddagger}$ &
                $66.9$ & $19.2$ & $3.5$ & $9.9$ & $\mathbf{1.3}^{\ddagger}$ & ${0.6}$
\\
ScatterBrain     & 
                $6.4$ & $2.1$ & $0.7$ &
                $58.8$ & $21.0$ & $5.2$ & $9.7$ & $2.6$ & $0.8$
\\
PointTrans     & 
                $2.9$ & $\mathbf{1.2}^{\ddagger}$ & $\mathbf{0.5}^{\ddagger}$ &
                $30.6$ & $\mathbf{18.3}^{\ddagger}$ & $4.4$ & $4.9$ & $\mathbf{1.3}^{\ddagger}$ & ${0.7}$
\\
FlatFormer     & 
                $5.7$ & $\mathbf{1.3}$ & $\mathbf{0.5}^{\ddagger}$ &
                $58.7$ & $19.8$ & $\mathbf{2.5}^{\ddagger}$ & $9.6$ & $1.6$ & $\mathbf{0.4}^{\ddagger}$
\\
\midrule
GCN     & 
                $2.0$ & $\mathbf{1.2}^{\ddagger}$ & $0.6$ & $20.9$ & $\mathbf{18.6}$ & ${4.1}$ & $3.4$ & $\mathbf{1.3}^{\ddagger}$ & $\mathbf{0.4}^{\ddagger}$
\\
DGCNN     & 
                $11.9$ & $1.6$ & $0.6$ &
                $124.0$ & $22.3$ & $4.1$ & $15.5$ & $1.9$ & $0.6$
\\
GravNet     & 
                $2.1$ & $\mathbf{0.8}^{\dagger}$ & $\mathbf{0.2}^{\dagger}$ &
                $21.9$ & $\mathbf{14.5}^{\dagger}$ & $\mathbf{2.4}^{\dagger}$ & $4.3$ & $\mathbf{0.7}^{\dagger}$ & $\mathbf{0.4}^{\ddagger}$
\\
GatedGNN     & 
                $2.4$ & $2.5$ & $1.2$ &
                $24.8$ & $23.8$ & $10.6$ & $3.9$ & $3.3$ & $1.5$
\\
\midrule
\proj     & 
            $5.0$ & $\mathbf{1.3}$ & $\mathbf{0.5}^{\ddagger}$ &
            $52.2$ & $19.9$ & $\mathbf{2.9}$ & $8.5$ & $\mathbf{1.3}^{\dagger}$ & $\mathbf{0.2}^{\dagger}$
\\
\bottomrule
\end{tabular}%
}
\vspace{-5mm}
\end{table}

\subsection{Implementation of Baselines \& Hyperparameter Tuning}
All baseline transformers are implemented following the same standard transformer architecture as above with the full self-attention module replaced with the corresponding proposed efficient attention modules, and GNNs are realized using the implementation from PyTorch Geometric~\cite{Fey_Fast_Graph_Representation_2019}.

For baseline transformers, as the number of trainable parameters and the architecture is fixed, we only need to tune method-specific hyperparameters, and we are to tune these hyperparameters in a (small) range of FLOPs that would not deviate too much (e.g., $\pm 10\%$ GFLOPs) if possible such that all baseline transformers are ensured to be with similar FLOPs for a fair comparison. For GNN baselines, we mainly follow the implementation from the authors' code and change the hidden dimensions to align the number of trainable parameters.
In the following, we describe in detail how each baseline is implemented and tuned, and Table~\ref{tab:speed} and Table~\ref{tab:memory} benchmark the computational speed, FLOPs, and GPU memory usage for the tuned baseline models.

\textbf{Basic Settings.} 
For all datasets and baselines, $\operatorname{Adam}$ optimizer~\cite{kingma2014adam} is used.
For the two Tracking datasets, the learning rate is tuned from $\{1e^{-2}, 1e^{-3}\}$, and is multiplied by a factor of $0.5$ every $500$ epochs. Any model will be early-stopped if there is no improvement in the validation set over $200$ consecutive epochs, and models can be trained for up to $2000$ epochs to ensure convergence. Models for these two datasets are set with $0.33$M trainable parameters for efficiency.
For the Pileup dataset, the learning rate is tuned from $\{1e^{-3}, 1e^{-4}\}$, and is multiplied by a factor of $0.5$ if there is no improvement in the validation set for $20$ epochs. For this dataset, models can be trained for up to $200$ epochs for convergence, and they are set with $0.31$M trainable parameters.

\textbf{\proj.} Denote $G$ the total number of desired buckets (i.e., the number of unique aux hash codes allowed) when obtaining AND hash codes, which is tuned from $\{10, 15, 20\}$ for Tracking-6k, from $\{100, 150, 200 \}$ for Tracking-60k, from $\{100, 120, 140 \}$ for Pileup-10k. And $B_{ij}$'s are generated randomly such that $\prod_{j=2}^{m_2} B_{ij}=G$. Note that $B_{ij}$'s do not have to be integers.

\textbf{Reformer~\cite{kitaev2020reformer}} is implemented via~\cite{pytorchreformer}.
Its hyperparameters are tuned from $\{(\text{Block Size}: 150,\, \text{\# Hash Tables}: 3), (\text{Block Size}: 100,\, \text{\# Hash Tables}: 3), (\text{Block Size}: 100,\, \text{\# Hash Tables}: 2)\}$.

\textbf{SMYRF~\cite{daras2020smyrf}} is implemented via~\cite{flyrepo}. Its hyperparameters are tuned from $\{(\text{Block Size}: 150,\, \text{\# Hash Tables}: 3), (\text{Block Size}: 100,\, \text{\# Hash Tables}: 3), (\text{Block Size}: 100,\, \text{\# Hash Tables}: 2)\}$.

\textbf{HyperAttn~\cite{han2023hyperattention}} is implemented via the author-provided code. Its hyperparameters are tuned from $\{(\text{Block Size}: 100,\, \text{Sample Size}: 300), (\text{Block Size}: 150,\, \text{Sample Size}: 200), (\text{Block Size}: 100,\, \text{Sample Size}: 200)\}$.

\textbf{Performer~\cite{choromanski2020rethinking}} is implemented via~\cite{pytorchperformer}. Its number of feature map dimensions is tuned from $\{ 150,\, 200,\, 250 \}$.

\textbf{FLT~\cite{choromanski2023learning}} is implemented via~\cite{pytorchperformer, fasttransformers}. Its hyperparameters are tuned from $\{(\text{\# Feature Maps for RPE}: 10,\, \text{\# Feature Maps for Attn}: 150), (\text{\# Feature Maps for RPE}: 10,\, \text{\# Feature Maps for Attn}: 100), (\text{\# Feature Maps for RPE}: 20,\, \text{\# Feature Maps for Attn}: 100)\}$.

\textbf{ScatterBrain~\cite{chen2021scatterbrain}} is implemented via~\cite{flyrepo}. Its hyperparameters are tuned from $\{(\text{Block Size}: 100,\, \text{\# Hash Tables}: 2,\, \text{\# Feature Maps}: 100), 
(\text{Block Size}: 100,\, \text{\# Hash Tables}: 3,\, \text{\# Feature Maps}: 50), 
(\text{Block Size}: 50,\, \text{\# Hash Tables}: 2,\, \text{\# Feature Maps}: 150)
\}$.

\textbf{FlatFormer~\cite{liu2023flatformer}} is implemented via the author-provided code to adapt to general point-cloud data. We follow its proposed architecture, which is a bit different from the standard transformer. We tune its ``window shape" by projecting points into each axis and partitioning each axis equally into $N$ parts. This $N$ is tuned from $\{ 20,\, 30,\, 40\}$ for Tracking-6k, $\{ 100,\, 150,\, 200\}$ for Tracking-60k, from $\{ 30,\, 40,\, 50\}$ for Pileup-10k.

\textbf{Other Baselines.} For Point Transformer~\cite{zhao2021point}, GCN~\cite{kipf2016semi}, and GravNet~\cite{qasim2019learning}, we directly adopt the implementation from PyTorch Geometric. For 
DGCNN~\cite{qu2020jet}, we follow the description in the paper and modify the implementation from PyTorch Geometric accordingly. For GatedGNN~\cite{li2023semi}, we use the author-provided code, which is also based on PyTroch Geometric. These methods do not have extra hyperparameters to tune, and we change their hidden dimensions accordingly to align the number of trainable parameters.

\begin{table}[t]
\centering
\caption{Computational cost of each module in HEPT for inference latency (ms).}
\vspace{-0.2cm}
\label{tab:inference_latency_bd}
\begin{small}
\begin{tabular}{lccc}
\toprule
Module & Tracking-6k & Tracking-60k & Pileup-10k \\
\midrule
Attn & $5.8\,(83\%)$ & $52.8\,(91\%)$ & $9.2\,(86\%)$ \\
Other & $1.2\,(17\%)$ & $5.1\,(9\%)$ & $1.5\,(14\%)$ \\
\midrule
Total & $7.0\,(100\%)$ & $57.9\,(100\%)$ & $10.7\,(100\%)$ \\
\bottomrule
\end{tabular}
\end{small}
\vspace{-0.5cm}
\end{table}
\begin{table}[t]
\centering
\caption{Computational cost of each module in HEPT for training latency (ms).}
\vspace{-0.2cm}
\label{tab:training_latency_bd}
\begin{small}
\begin{tabular}{lccc}
\toprule
Module & Tracking-6k & Tracking-60k & Pileup-10k \\
\midrule
Loss+Backward & $330\,(97.6\%)$ & $2248\,(97.2\%)$ & $28.2\,(70\%)$ \\
Attn & $6.6\,(2.0\%)$ & $57.1\,(2.5\%)$ & $10.4\,(26\%)$ \\
Other & $1.3\,(0.4\%)$ & $6.2\,(0.3\%)$ & $1.7\,(4\%)$ \\
\midrule
Total & $338\,(100\%)$ & $2312\,(100\%)$ & $40.3\,(100\%)$ \\
\bottomrule
\end{tabular}
\end{small}
\vspace{-0.5cm}
\end{table}

\subsection{Implementation of Numerical Experiments}\label{app:numerical}
The numerical experiments conducted in Sec.~\ref{sec:numerical} generate $n=30,000$ points uniformly distributed across a 2D square with a side length of $10$. 
To model local inductive bias, each point interacts only with its $64$ nearest neighbors, approximating a ground-truth kernel value of $\exp \left(-\frac{1}{2}\|\bm{x}-\bm{y}\|^2\right)$ (we select this kernel because it is the well-known Gaussian kernel~\cite{seeger2004gaussian} and aligns with our proposed attention kernel, but our theoretical results are not limited to this kernel). Points beyond this neighborhood have a kernel value of $0$. 

Then, E$^2$LSH is utilized for approximation. Given a budget of FLOPs $F$, OR-only LSH approximates the kernel values by setting the number of hash functions per table to $1$ and searching the bucket size and the number of hash tables to find its optimized approximation error $\epsilon$ in this point cloud system; OR \& AND LSH searches the bucket size, the number of hash tables, and the number of hash functions per table, to obtain the optimized error for a given number of FLOPs.

The bucket size in E$^2$LSH is determined by adjusting the quantization term $r$ (see Sec.~\ref{sec:prelim}), which is searched from $0.01$ to $5$ with a step size of $0.05$. If searched, both the number of hash tables and the number of hash functions per table are searched from $1$ to $20$, with a step size of $1$.

\section{Latency Breakdown}
Table~\ref{tab:inference_latency_bd} and Table~\ref{tab:training_latency_bd} evaluate the computational cost of each module in HEPT using the same checkpoints from Table~\ref{tab:speed}. We can see that the majority of time is spent on attention computation during inference. 
On the other hand, during training, the computation of loss and gradient backpropagation dominates the total running time for the tasks considered in this work.

\section{Theoretical Results}

In this section, we provide the proof omitted in the main text. Recall the following settings for our analysis:

\defKernel*
\defSparsity*

Note that the point cloud system $\mathcal{C}$ may be a given deterministic or sample from a distribution $\mathcal{C}\sim \mathbb{P}$. In the latter case, we slightly abuse the notation by still using $\phi$ to denote the distance density function defined in Assumption~\ref{def:sparsity} while after the expectation over $\mathbb{P}$, $\mathbb{E}_{\mathbb{P}}(\phi)$.

To simplify our notation, we denote $k_s(z)=\mathbb{E}_{\bm{x},\bm{y}\in\mathcal{C}}[k_s(\bm{x}-\bm{y}) \mid  \|\bm{x}-\bm{y}\|_2=z]$ and $k_s^2(z)=\mathbb{E}_{\bm{x},\bm{y}\in\mathcal{C}}[k_s^2(\bm{x}-\bm{y}) \mid \|\bm{x}-\bm{y}\|_2=z]$ in the following subsections.

\subsection{Proof of Theorem~\ref{theorem:RFF}}

\theoremRFF*
\begin{proof}
    Denote $\widehat{k}_s(\bm{x}, \bm{y}) = \psi(\bm{x})^\top\psi(\bm{y})$.
    Since $\psi(\bm{x}) =  \sqrt{\frac{2}{{D}}}\Big[\sin (\bm{w}_1^{\top} \bm{x}),  \cos (\bm{w}_1^{\top} \bm{x}), \ldots, \sin (\bm{w}_{D/2}^{\top} \bm{x}), \cos (\bm{w}_{D/2}^{\top} \bm{x})\Big]^{\top}$ with $\bm{w}_i \stackrel{i i d}{\sim} k_s^*(\bm{w})$, its expected squared error w.r.t. $\bm{w}$~\cite{sutherland2015error} is
    $$
    \operatorname{MSE}_{k_s}\left(\widehat{k}_s(\bm{x}, \bm{y})\right)=\mathbb{E}_{\bm{w}}\left[\left(\widehat{k}_s(\bm{x}, \bm{y})-k_s(\bm{x}, \bm{y})\right)^2\right] = \frac{1}{D}(1+ k_s(2(\bm{x}-\bm{y})) - 2k_s(\bm{x}-\bm{y})^2).
    $$
    Therefore, the squared error averaged over all point pairs in the system is
    $$
    \epsilon = \mathbb{E}_{z\sim\phi(z)} \left[  \frac{1}{D}\left(1+ k_s(2z) - 2k_s(z)^2\right) \right]
    = \frac{1}{D}\left(1 + \mathbb{E}_{z\sim\phi(z)}\left[k_s(2z)\right] - 2\mathbb{E}_{z\sim\phi(z)}\left[k_s^2(z)\right] \right).
    $$
    Since $\int_{0}^{s} \phi(z)dz \sim \tilde{\mathcal{O}}(\frac{1}{n})$ (Assumption~\ref{def:sparsity}) and $k_s(z) \in [0, 1]$, we have $\epsilon=\Theta(\frac{1}{D})$.

    To use this RFF to approximate the attention mechanism $ \bm{A} \bm{V}$, where $\bm{V}\in \mathbb{R}^{n \times d}$ is the value matrix and $\bm{A} \in \mathbb{R}^{n \times n}$ is the unnormalized attention matrix with each entry $\bm{A}_{\bm{x},\bm{y}} = k_s(\bm{x},\bm{y})$, 
    we first obtain $\bm{X}', \bm{Y}'\in \mathbb{R}^{n \times D}$ with each row given by $\psi(\bm{x})$ and $\psi(\bm{y})$, respectively, which requires $nD(2d-1) + nd = \Theta(nDd)$ FLOPs. Then,
    $\bm{A}\bm{V} \approx \bm{X}' (\bm{Y}'^\top \bm{V})$ needs $Dd(2n-1) + nd(2D-1)= \Theta(nDd)$ FLOPs. Note that when using RFF, it is important to first compute $\bm{Y}'^{\top}\bm{V}$ to avoid the complexity of $n^2$ for computing $\bm{X}' \bm{Y}'^\top$. Thus, the total FLOPs required are $F = \Theta(ndD)$.

    Therefore, with $\epsilon=\Theta(\frac{1}{D})$ and $F = \Theta(ndD)$, we have $\epsilon = \Theta\left( \frac{nd}{F} \right)$.
\end{proof}

\subsection{Proof of Theorem~\ref{theorem:LSH}}
\begin{lemma}\label{lemma:probBound}
Consider the collision probability $p_r(z)$ in E$^2$LSH, employing the hash function $h_{\bm{a}, b}(\bm{x}) = \lfloor \frac{\bm{a}\cdot \bm{x} + b}{r} \rfloor$, for two points $\bm{x}, \bm{y} \in \mathbb{R}^d$ with distance $z=\|\bm{x}-\bm{y}\|_2$. This probability can be bounded as follows:

For $z < r$,
$$
1 - \sqrt{\frac{2}{\pi}} \frac{z}{r} \leq p_r(z) \leq 1-\sqrt{\frac{1}{2 \pi}} \frac{z}{r}.
$$
For $ z \geq r$,
$$
\frac{\sqrt{2}}{3 \sqrt{\pi}} \frac{r}{z} \leq p_r(z) \leq \frac{1}{\sqrt{2 \pi}} \frac{r}{z}.
$$
Consequently, the expected collision probability for pairs of points in the point cloud systems described in Assumption~\ref{def:sparsity} can be bounded as:
$$
\mathbb{E}_{z \sim \phi(z)}\left[p_r(z)\right] \leq \int_0^r\left(1-\sqrt{\frac{1}{2 \pi}} \frac{z}{r}\right) \phi(z) d z+\sqrt{\frac{1}{2 \pi}} r \int_r^{\infty} \frac{1}{z} \phi(z) d z.
$$
\end{lemma}
\begin{proof}
    With hash functions from E$^2$LSH, the collision probability for two distinct points $\bm{x}, \bm{y} \in \mathbb{R}^d$ with distance $z=\|\bm{x}-\bm{y}\|_2$ is~\cite{datar2004locality}:
    $$
    p_r(z)=\operatorname{P}\left[h_{\bm{a}, b}(\bm{x})=h_{\bm{a}, b}(\bm{y})\right] = \int_0^r \frac{1}{z} f_2\left(\frac{t}{z}\right)\left(1-\frac{t}{r}\right) dt,
    $$
    where $f_2(\cdot)$ denotes the PDF of the absolute value of the $2$-stable distribution.
    Thus, 
    $$
    p_r(z)=\operatorname{Erf}\left(\frac{r}{\sqrt{2} z}\right)-\sqrt{\frac{2}{\pi}} \frac{z}{r}\left(1-\exp \left(-\frac{r^2}{2 z^2}\right)\right),
    $$
    where $\operatorname{Erf}(u)=\frac{2}{\sqrt{\pi}} \int_{0}^{u} e^{-t^2} \mathrm{~d} t$.

    Let $u = \frac{r}{\sqrt{2}z}$, $p_r(z) = f(u) = \operatorname{Erf}(u) - \frac{1}{u \sqrt{\pi}}\left(1-\exp \left(-u^2\right)\right)$. We are to bound $f(u)$ for $u > \frac{1}{\sqrt{2}}$ and $0 < u \leq \frac{1}{\sqrt{2}}$.

   \textbf{Bounding $f(u)$ for $u > \frac{1}{\sqrt{2}}$.} Consider $g(u) = f(u) - ( 1 - \frac{1}{u\sqrt{\pi}} )$. Since $g'(u) = \frac{-\exp(-u^2)}{\sqrt{\pi} u^2} < 0$, $g(\frac{1}{\sqrt{2}})>0$, and $\lim_{u\rightarrow \infty} g(u) = 0$, we have $f(u) \geq 1-\frac{1}{u \sqrt{\pi}}$. 
   Now, consider $g(u) = f(u) - ( 1 - \sqrt{\frac{1}{\pi}} \frac{1}{2u} )$.
   Since $g'(u) = \frac{-2+\exp(u^2)}{2 \exp(u^2) \sqrt{\pi} u^2}$, it has a local minima at $u = \sqrt{\ln 2}$. Because $g'(u) < 0$ when $\frac{1}{\sqrt{2}} < u < \sqrt{\ln 2}$, $g'(u) > 0$ when $u > \sqrt{\ln 2} $, $g(\frac{1}{\sqrt{2}}) < 0$, and $\lim_{u\rightarrow \infty} f(u) = 0$, we have $f(u)\leq1-\sqrt{\frac{1}{\pi}} \frac{1}{2 u}$.

   \textbf{Bounding $f(u)$ for $0 < u \leq \frac{1}{\sqrt{2}}$.} Consider $g(u) = f(u) - \frac{2}{3\sqrt{\pi}}u$. Since $g'(u)=\frac{3 - (2u^2+3/\exp(u^2))}{3 \sqrt{\pi}u^2} >0$ when $0 < u \leq \frac{1}{\sqrt{2}}$, $f(\frac{1}{\sqrt{2}}) > 0$, and $\lim_{u\rightarrow0}f(u)=0$, we have $f(u) \geq \frac{2}{3 \sqrt{\pi}} u$. Now, consider $g(u) = f(u) - \frac{1}{\sqrt{\pi}}u$. Since $g'(u) = \frac{1-(u^2+\exp(-u^2))}{\sqrt{\pi}u^2} < 0$ when $0<u \leq \frac{1}{\sqrt{2}}$, $f(\frac{1}{\sqrt{2}}) < 0$, and $\lim_{u\rightarrow0}f(u)=0$, we have $f(u) \leq \frac{1}{\sqrt{\pi}}u$.

   Therefore, substituting the $u$ back, we yield the above bounds for $p_r(z)$, which directly gives the bound for $\mathbb{E}_{z \sim \phi(z)}\left[p_r(z)\right]$.
\end{proof}

\begin{lemma}\label{lemma:lshFLOPs}
    Consider using E$^2$LSH to approximate the attention mechanism $ \bm{A} \bm{V}$, where $\bm{V}\in \mathbb{R}^{n \times d}$ is the value matrix and $\bm{A} \in \mathbb{R}^{n \times n}$ is the unnormalized attention matrix with each entry $\bm{A}_{\bm{x},\bm{y}} = k_s(\bm{x},\bm{y})$. If performing OR-only LSH with $m_1$ hash functions for the approximation, the required FLOPs are $F = \Theta\left(m_1 d n^2 \mathbb{E}_{z \sim \phi(z)}\left[p_r(z)\right]  + m_1 nd \right)$. If constructing $m_1$ hash tables (OR LSH), with each having $m$ hash functions (AND LSH) for the approximation, the required FLOPs are $F=\Theta\left(m_1 d n^2 \mathbb{E}_{z \sim \phi(z)}\left[p_r(z)^m\right]+m_1 n d m\right)$.
\end{lemma}
\begin{proof}
    First, consider performing OR-only LSH with $m_1$ hash functions.
    Let $\mathbb{E}_{z \sim \phi(z)}[p_r(z)]$ denote the expected collision probability for pairs of points in the point cloud system described in Assumption~\ref{def:sparsity}.

    \textbf{1. Obtaining Hash Codes \& Computing Kernel Values.} With $n$ points in the system, it needs $n(2d-1)$ FLOPs to obtain hash code from each hash function. Then, $\binom{n}{2}$ point pairs result in $\binom{n}{2} \mathbb{E}_{z \sim \phi(z)}[p_r(z)]$ expected number of collisions (pairs of points with the same hash code) from each of the $m_1$ hash functions. To compute $k_s$ for all collided pairs, it requires $\Theta(d \binom{n}{2} \mathbb{E}_{z \sim \phi(z)}[p_r(z)])$ FLOPs.

    \textbf{2. Approximating the Attention Mechanism.} If one hash function results in $B$ buckets and in each bucket $b$ there are $o_b$ collisions (pairs of points with the same hash code) and $n_b$ unique points, then $\sum_{b\in[B]} o_b = \binom{n}{2} \mathbb{E}_{z\sim \phi(z)}[p_r(z)]$, $\binom{n_b}{2} = o_b$, and $\sum_{b\in[B]} n_b = n$. 
    Computing $\widehat{\bm{A}}\widehat{\bm{V}}$ for each bucket $b$ needs
    $d n_b (2n_b - 1)$ FLOPs, since $\widehat{\bm{A}}$ is $n_b \times n_b$ and $\widehat{\bm{V}}$ is $n_b \times d$. So, it needs $4d \binom{n}{2}  \mathbb{E}_{z\sim \phi(z)}[p_r(z)] + nd$ FLOPs for all buckets since $\sum_{b\in[B]} d n_b (2n_b - 1) = \sum_{b\in[B]} 2(n_b^2 - n_b)d + n_b d = \sum_{b\in[B]} 4 o_bd + n_bd $. 

    \textbf{3. Combing $m_1$ Hash Results.} The above process is repeated for $m_1$ times, and $(m_1-1)nd$ extra FLOPs are needed to combine the $m_1$ hash results. Therefore, the total FLOPs required is $F = \Theta(m_1 dn^2 \mathbb{E}_{z \sim \phi(z)}\left[p_r(z)\right] + m_1 nd)$ if OR-only LSH is performed with $m_1$ hash functions.
    
    Now, consider constructing $m_1$ hash tables (OR LSH), each with $m$ hash functions (AND LSH). 
    This results in collision probability $p_r(z)^m$ and now needs $\Theta(m_1 ndm)$ FLOPs to obtain all hash codes and combine all results. Then, the total FLOPs required are $F=\Theta\left(m_1 d n^2 \mathbb{E}_{z \sim \phi(z)}\left[p_r(z)^m\right]+m_1 n d m\right)$.
\end{proof}

\theoremLSH*
\begin{proof}
Consider performing OR-only E$^2$LSH with $m_1$ hash functions to approximate $k_s$ in the point cloud system described in Assumption~\ref{def:sparsity}. 
The resulting squared error averaged over all point pairs in the system is then $\epsilon=\mathbb{E}_{z \sim \phi(z)}\left[(1 - p_r(z))^{m_1} k_s^2(z)\right]$, where $p_r(z)$ is the collision probability of E$^2$LSH.

First, the complexity $F$ has a natural lower bound: Due to Lemma~\ref{lemma:lshFLOPs} and the bound of $p_r(z)$ in Lemma~\ref{lemma:probBound}, $$F=\Theta(d m_1 n^2 \mathbb{E}_{z \sim \phi(z)}\left[p_r(z)\right] +m_1nd)\geq  \Theta\left(d m_1 n^2 (\int_0^r \phi(z)(1-\sqrt{\frac{2}{\pi}}\frac{z}{r}) dz + \int_r^1 \phi(z)\frac{\sqrt{2}}{3 \sqrt{\pi}} \frac{r}{z} dz)\right)$$$$\geq\Theta\left(d m_1 n^2 (\int_0^r \phi(z)dz + r\int_r^1 \phi(z)dz)\right)\geq \Theta\left(d m_1 n^2r\right).$$

\textbf{Upper bound:} Here, we first show that for some positive $c_3$, $$
\epsilon = \tilde{\mathcal{O}}( \exp \left(-\frac{c_3 F}{d n^2 s}\right) \frac{1}{n}). 
$$ 


We are only interested in the regime with limited complexity where $F=O(dn^2s)$. Otherwise, the above error is almost 0 and the complexity is already super linear because $s\gg \frac{1}{n}$ in general (see the discussion in Sec.\ref{sec:LIB}). Since for OR-only E$^2$LSH, $F\geq \Theta\left(d m_1 n^2r\right)$, 
this means, in practice, to satisfy $F=O(dn^2s)$, we will set $r \leq s$. To better understand this point, if we set $r > s$, intuitively, one single hash function is sufficient to put points within distance $s$ into the same hash bucket with high probability, which is able to compute the attention weights accurately. However, in this case, there will be $n\sqrt{s}$ many points mapped into the same bucket, which gives  complexity as much as $\Omega\left(d m_1 n^2s\right)$. This can be understood from the lower bound $F\geq \Theta\left(d m_1 n^2r\right)$:x when $r>s$, the above bound of $F$ implies $F\geq \Theta\left(d m_1 n^2r\right)\geq \Theta\left(d m_1 n^2s\right)$.

Let us next suppose $r\leq s$. With Lemma~\ref{lemma:probBound}, we have 
$$
\epsilon \leq   ( 1 - c''\frac{r}{s} )^{m_1}  \mathbb{E}_{z \sim \phi(z)}\left[ k_s^2(z)\right]\leq \exp \left(-\frac{c'' m_1 r}{s}\right)  \mathbb{E}_{z \sim \phi(z)}\left[ k_s^2(z)\right],
$$
where $c'' = \frac{\sqrt{2}}{3 \sqrt{\pi}}$.

By assumption, there exists $r$ such that $\int_{0}^r \phi(z) dz \leq c_1 r$ and $\int_{r}^\infty \frac{1}{z} \phi(z) dz \leq c_2$ for some positive constants $c_1$ and $c_2$, we have $\mathbb{E}_{z \sim \phi(z)}\left[p_r(z)\right] \leq c'r$, where $c^{\prime}=c_1+c_2 \sqrt{\frac{1}{2 \pi}}$. Then, since $ F=\Theta(d m_1 n^2 \mathbb{E}_{z \sim \phi(z)}\left[p_r(z)\right] +m_1nd) $ (Lemma~\ref{lemma:lshFLOPs}), we have $F = \mathcal{O} (d m_1 n^2c'r) $. Thus,
$$
\epsilon \leq \exp \left(-\frac{c_3 F}{d n^2 s}\right) \mathbb{E}_{z \sim \phi(z)}\left[ k_s^2(z)\right],
$$
where $c_3$ is a positive constant depending on $c_1$ and $c_2$.

Since $\int_0^s \phi(z) d z \sim \tilde{\mathcal{O}}\left(\frac{1}{n}\right)$ (Assumption~\ref{def:sparsity}) and $k_s(z) \in  [0,1]$, we have
$$
\epsilon = \tilde{\mathcal{O}}( \exp \left(-\frac{c_3 F}{d n^2 s}\right) \frac{1}{n}).
$$

Now, we lower bound $\epsilon$. 

\textbf{Lower bound:} Next, we show that there exists a point cloud system $\mathcal{C}$ and the kernel $k_s$ that satisfy the assumption, while for some positive $c_5$,
$$ 
\epsilon = \Omega( \exp \left(-\frac{c_5 F}{d n^2 s}\right) \frac{1}{n}).
$$

We consider a very common case when $k_s(z) = 1$ when $z\in[0,s]$ and the point cloud is uniformly allocated in the unit ball. In this case, $\phi(z)\propto z^{d-1}$, and $s\sim \frac{1}{n^{1/d}}$ as shown in Sec.~\ref{sec:LIB}. It is easy to verify that if $r$ satisfies $r\leq s$, the conditions in the theorem statement are true: This is because $\int_{0}^r \phi(z) dz \lesssim r^d < r$ and $\int_{r}^1 \frac{1}{z}\phi(z) dz \leq \frac{d}{d-1}$. Again, we only focus on the regime $r\leq s$. It is not hard to show that when $r\gg s$, the lower bound is even higher than above. 

When $r\leq s$, the above lower bound of $F$ already gives $F = \Omega (d m_1 n^2r)$. With Lemma~\ref{lemma:probBound} and $k_s(z)=1$ for $z\leq s$, we have 
$$
\epsilon =  \int_0^s (1-p_r(z))^{m_1}\phi(z)dz \geq \int_{0}^r (1-p_r(z))^{m_1}\phi(z)dz+ \int_{r}^s (1-p_r(z))^{m_1}\phi(z)dz $$
$$\geq \Theta\left(\int_0^r\left(\frac{1}{\sqrt{2\pi}} \frac{z}{r}\right)^{m_1} z^{d-1} dz+ \int_r^s\left(1-\frac{1}{\sqrt{2\pi}} \frac{r}{z}\right)^{m_1} z^{d-1} dz \right)  \geq \Theta(r^d (\frac{1}{\sqrt{2\pi}})^{m_1} + (1-\frac{1}{\sqrt{2\pi}})^{m_1}(s^d-r^d))
$$
$$= \Theta(\exp(-c_5m_1)\cdot s^d) = \Omega( \exp \left(-\frac{c_5 F}{d n^2 r}\right)\cdot\frac{1}{n}), \quad \text{where $c_5$ is a positive constant.}
$$
This completes the proof.
\end{proof}

\subsection{Proof of Theorem~\ref{theorem:ANDLSH}}
\theoremANDLSH*
\begin{proof}
Consider using E$^2$LSH to construct $m_1$ hash tables (OR LSH), each with $m$ hash functions (AND LSH). To approximate $k_s$ in the point cloud systems described in Assumption~\ref{def:sparsity}, 
the resulting squared error averaged over all point pairs in the system 
is $\epsilon=\mathbb{E}_{z \sim \phi(z)}\left[\left(1-p_r(z)^{m}\right)^{m_1} k_s^2(z)\right]$ and the FLOPs required are $F= \Theta( m_1 d n^2 \mathbb{E}_{z \sim \phi(z)}\left[p_r(z)^{m}\right]  + m_1ndm ) $ (Lemma~\ref{lemma:lshFLOPs}).

Pick the smallest $m$ that satisfies 
$\int_{0}^{r}\phi(z)dz = \tilde{\mathcal{O}}(\frac{1}{n})$ and 
$\int_{r}^{\infty}\phi(z)\frac{r^{m}}{z^{m}}dz\leq \int_{0}^{r}(\sqrt{2\pi} - \frac{z}{r})^{m}\phi(z)dz$, where $r=ms$, combined with Lemma~\ref{lemma:probBound}, then we have 
$$
F = \mathcal{O} \left( m_1 d n^2 \int_0^{m s}\left(1-\frac{1}{\sqrt{2 \pi}} \frac{z}{m s}\right)^{m} \phi(z) d z+m_1n d m \right) = \mathcal{O}  \left( m_1 d n^2 \int_0^{m s}\phi(z) d z+m_1n d m \right).
$$

Similarly, with $r=ms$ and Lemma~\ref{lemma:probBound}, we have
$$
\epsilon \leq \left( 1 -  \left( 1 - c' \frac{s}{ms} \right)^m  \right)^{m_1} \mathbb{E}_{z \sim \phi(z)}\left[ k_s^2(z)\right] \leq  \exp(-c'' m_1) \mathbb{E}_{z \sim \phi(z)}\left[ k_s^2(z)\right],
$$
where $c'=\sqrt{\frac{2}{\pi}}$ and $c''=1-c'$.

Therefore, 
$$
\epsilon = \mathcal{O}\left(   \exp \left(-\frac{c'' F}{d n^2 \int_0^{m s} \phi(z) d z+n d m}\right)  \mathbb{E}_{z \sim \phi(z)}\left[ k_s^2(z)\right] \right).
$$

Since $\int_0^r \phi(z) d z=\tilde{\mathcal{O}}\left(\frac{1}{n}\right)$, $\int_0^s \phi(z) d z \sim \tilde{\mathcal{O}}\left(\frac{1}{n}\right)$, and $k_s(z) \in[0,1]$, we yield
$$
\epsilon = \tilde{\mathcal{O}}\left(\exp \left(- \frac{c_4F}{dn( \operatorname{polylog}(n) + m)}\right)\frac{1}{n}\right),
$$
where $c_4$ is some positive constant.
\end{proof}

\end{document}